%% file: main.tex
\documentclass[journal]{IEEEtran}
\usepackage{amsmath,amsfonts}
\usepackage{algorithmic}
\usepackage{array}
\usepackage{textcomp}
\usepackage{stfloats}
\usepackage{url}
\usepackage{verbatim}
\usepackage{graphicx}
\usepackage{booktabs}
\usepackage{nicefrac}
\usepackage{microtype}

\usepackage{enumitem}

\usepackage{amssymb}
\usepackage{mathtools}
\usepackage{amsthm}
\usepackage{bbm}
\usepackage{mathrsfs}

\newtheorem{theorem}{Theorem}[section]
\newtheorem{proposition}[theorem]{Proposition}
\newtheorem{lemma}[theorem]{Lemma}
\newtheorem{corollary}[theorem]{Corollary}
\theoremstyle{definition}

\newtheorem{assumption}[theorem]{Assumption}
\theoremstyle{remark}

\usepackage[colorlinks=true,linkcolor=black,citecolor=blue,urlcolor=blue,backref=page]{hyperref}

\usepackage{xcolor}
\usepackage{my_symbol}
\usepackage[most]{tcolorbox}
\usepackage{lipsum}
\usepackage{subcaption}

\newtcolorbox{setup}{
  colframe=cyan!30,
  colback=cyan!5,
  coltitle=black,
  fonttitle=\bfseries,
  title=Setup
}

\DeclareMathOperator*{\argmin}{arg\,min}

\usepackage{multirow}

\usepackage{tikz}
\newcommand*\circled[1]{\tikz[baseline=(char.base)]{
            \node[shape=circle,draw,inner sep=2pt] (char) {#1};}}

\def\BibTeX{{\rm B\kern-.05em{\sc i\kern-.025em b}\kern-.08em
    T\kern-.1667em\lower.7ex\hbox{E}\kern-.125emX}}
\usepackage{balance}

\begin{document}
\title{Graph Semi-Supervised Learning for Point Classification on Data Manifolds
}
\author{Caio F. Deberaldini Netto*, Zhiyang Wang\textsuperscript{\textdagger} and Luana Ruiz*
\thanks{*Department of Applied Mathematics and Statistics, Mathematical Institute for Data Science (MINDS), Data Science and Artificial Intelligence Institute (DSAI), Johns Hopkins University, Baltimore, USA. E-mail: \{cnetto1, lrubini1\}@jh.edu
\textsuperscript{\textdagger}Halıcıoğlu Data Science Institute, University of California San Diego, California, USA. E-mail: zhw135@ucsd.edu
}}

\markboth{IEEE Transactions on Signal Processing (Submitted)}%
{How to Use the IEEEtran \LaTeX \ Templates}

\maketitle

\begin{abstract}
We propose a graph semi-supervised learning framework for classification tasks on data manifolds. Motivated by the manifold hypothesis, we model data as points sampled from a low-dimensional manifold $\mathcal{M} \subset \mathbb{R}^F$. The manifold is approximated in an unsupervised manner using a variational autoencoder (VAE), where the trained encoder maps data to embeddings that represent their coordinates in $\mathbb{R}^F$. A geometric graph is constructed with Gaussian-weighted edges inversely proportional to distances in the embedding space, transforming the point classification problem into a semi-supervised node classification task on the graph. This task is solved using a graph neural network (GNN).
Our main contribution is a theoretical analysis of the statistical generalization properties of this data-to-manifold-to-graph pipeline. We show that, under uniform sampling from $\mathcal{M}$, the generalization gap of the semi-supervised task diminishes with increasing graph size, up to the GNN training error. Leveraging a training procedure that resamples a slightly larger graph at regular intervals during training, we then show that the generalization gap can be reduced even further, vanishing asymptotically. Finally, we validate our findings with numerical experiments on image classification benchmarks, demonstrating the empirical effectiveness of our approach.
\end{abstract}


\input{Sections/introduction}
\input{Sections/background}
\input{Sections/method}
\input{Sections/experiments}
\input{Sections/discussions}

\bibliographystyle{ieeetr}
\bibliography{myIEEEabrv,refs,bib_cumulative}

\input{Sections/appendix}

\end{document}

%% file: Sections/introduction.tex
\section{Introduction}

Graph neural networks (GNNs) have demonstrated excellent performance on graph-structured data, with successful application examples ranging from molecular biology and chemistry \cite{gilmer2017neural}, through network science \cite{gao2023recsys_graph_survey}, to natural language processing 
\cite{santoro2017simple}. 
Though results are remarkable across the board, perhaps the most impactful applications of GNNs are in the semi-supervised setting. This is evidenced by the three most cited and benchmarked graph datasets in Papers with Code, which are all semi-supervised \cite{paperswithcode_graph_datasets}. In graph semi-supervised learning, there is only one graph, and we assume we know information about every node in the form of node features. We further assume nodes belong to different classes, but most nodes' classes are unknown. The goal of semi-supervised learning is to use supervision over this partial set of labeled nodes to predict the classes of unlabeled nodes.

Several GNN properties justify their superior performance in semi-supervised learning. For instance, permutation equivariance ensures that patterns learned on certain graph substructures are automatically replicated on regions where these substructures are repeated \cite{ruiz2020gnns,kipf17-classifgcnn,maron2018invariant,battaglia2018relational}, while stability ensures these patterns are replicated even on slightly deformed but still similar regions \cite{gama19-stability,gama2019stabilityscat}.
A particularly important property is \textit{transferability}, which is the ability of a trained model to retain predictive performance across different graphs belonging to the same family, e.g., sampled from the same random graph model, or converging to the same graph limit \cite{ruiz20-transf,keriven2020convergence,levie2019transferability}. The transferability property supports the superior performance of GNNs in graph semi-supervised learning because limited supervision means the GNN is only optimized over a subgraph, which can be seen as a sample from the original graph.

GNN transferability was well-studied in the case of \textit{geometric graphs}, where nodes represent samples from a continuous non-Euclidean space, i.e., a manifold, and edges capture proximity relationships. 
It was also demonstrated empirically in graph families with \emph{explicit geometric structure}, with notable use cases including classification of point clouds and meshes with different resolutions \cite{bronstein17-geomdeep,wang2019dynamic} and GNN-based path planners able to operate in unseen regions of the same manifold \cite{cervino2023learning}.
At the same time, many other types of high-dimensional data exhibit \emph{intrinsic low-dimensional geometry}; this is the so-called \textit{manifold hypothesis} \cite{bengio2013representation}. Natural image distributions, for example, concentrate around a low-dimensional manifold embedded in high-dimensional space \cite{zhu2016generative}. 
The same holds for physical systems, where governing equations impose smooth, low-dimensional constraints on the data \cite{Sroczynski2018DataDriven}. 

Considering that high-dimensional datasets have intrinsic geometry---in the vein of graphs sampled from a manifold---, a pertinent question is:  

\begin{center}
\textit{Can we leverage the \textbf{intrinsic geometric structure of data} together with the \textbf{transferability property of GNNs} to improve predictive performance on high-dimensional data?} 
\end{center}

This is the premise of our work. Focusing on image classification specifically, instead of treating classification as a standard pointwise learning task we \textit{construct a geometric graph from the data} and use \textit{graph-based semi-supervised learning} to improve generalization and predictive performance. The idea is that if the graph accurately captures the manifold structure, then GNNs should be able to exploit this geometry, resulting in models that better capture the data distribution and are able to make better predictions.

\subsection{Contributions}

We verify this premise with the following contributions:

\

\textbf{Reformulating point classification as graph semi-supervised learning.} We approximate the data manifold using a geometric graph constructed from variational autoencoder embeddings. This transforms the original classification problem into a \textit{semi-supervised node classification task} on the graph.

\textbf{Generalization analysis of GNNs on geometric graphs.} We prove that, when the graph is sampled from a manifold, the semi-supervised \textit{generalization gap has an upper bound that decreases as the graph size increases}. However, this bound depends on the loss over the entire manifold, which is impractical since we only have access to finite samples.

\textbf{A more practical generalization bound.} We refine our bound to show that the generalization gap depends only on the \textit{semi-supervised training loss on the graph}, rather than the intractable loss on the full manifold. This makes the bound more relevant in practice, albeit looser.

\textbf{A training procedure to improve generalization.} We leverage an algorithm that \textit{increases the graph size over time during training}, ensuring that the learning process remains closer to the underlying manifold. We show that this approach leads to a more practical and smaller generalization gap.

\textbf{Empirical validation.} We validate our theoretical findings with experiments on image classification benchmarks, demonstrating that GNN-based learning on geometric graphs from image manifolds improves statistical generalization and predictive performance by leveraging data geometry.

This work provides both theoretical and empirical support for using GNN-based semi-supervised learning on geometric graphs as a principled approach for classification on data manifolds. By structuring learning around the \textit{geometry of the data}, we show that GNNs can generalize better, even with limited labeled samples.

\section{Related work}

\noindent \textbf{Graph neural networks.} GNNs are deep convolutional architectures tailored to graphs \cite{scarselli2008graph}. Each layer of a GNN consists of a graph convolution, an extension of the convolution to the graph domain, followed by a pointwise nonlinearity acting node-wise on the graph \cite{du2018graph,ruiz2020gnns,kipf17-classifgcnn,defferrard17-cnngraphs,bruna14-deepspectralnetworks}. The nice theoretical properties of GNNs---invariances, stability, locality---are inherited from graph convolutions \cite{ruiz19-inv,gama19-stability}, which, similarly to time or image convolutions, operate by means of a shift-and-sum operation where shifts are implemented as graph diffusions via local node-to-neighbor exchanges \cite{segarra17-linear,gama18-gnnarchit,du2018graph}. One can also encounter GNNs described in terms of local aggregation functions \cite{xu2018how,hamilton2017inductive,gilmer2017neural}, which may be seen as particular cases of GNNs that use graph convolutional filters of order one [c.f. \cite{ruiz2020gnns}].

\noindent \textbf{GNN transferability.} A popular graph limit model for studying the transferability property of GNNs are graphons \cite{magner2020power,avella2018centrality,eldridge2016graphons,borgs2017graphons}, which are measurable functions representing limits of sequences of dense graphs \cite{borgs2008convergent,lovasz2012large}. A series of works introduced graphon convolutions and proved asymptotic convergence of graph to graphon convolutions, which in turn implies convergence of GNNs to graphon NNs \cite{ruiz2019graphon,ruiz20-transf,ruiz2020graphonsp,ruiz2021graphon-filters,ruiz2021transferability,cervino2021increase}. The non-asymptotic statement of this result, appearing in \cite{ruiz2021transferability}, implies the so-called transferability property of GNNs: GNNs with fixed weights can be transferred across graphs converging to or sampled from the same graphon with bounded error. 

Yet, graphons have an important limitation, which is that unless additional assumptions are made on the node sample space these models do not encode geometry, i.e., there is no clear embedding of the graph nodes in some continuous space. This is partially addressed by \cite{bronstein17-geomdeep} and \cite{levie2019transferability}, which consider graphs sampled from generic topological spaces. 
However, in these works graphs and graph signals are obtained by application of a generic sampling operator whose error is assumed bounded without taking into account the specific topology. Unlike these approaches, in this paper we particularize the topology to an embedded submanifold of some ambient Euclidean space. 
Endowing this manifold with the uniform measure, we can sample graphs by sampling points from the manifold uniformly and connecting them with edges depending on these points' distances in ambient space.

\noindent \textbf{The manifold hypothesis.} The ``manifold hypothesis'' posits that high-dimensional data lies on manifolds of lower dimension \cite{bengio2013representation}. Mathematically, this hypothesis was formalized in works such as \cite{narayanan2010sample,fefferman2016testing}, which developed conditions for when data on or near a smooth manifold can be accurately modeled and learned via sample complexity analysis and new test statistics. In practice, the manifold hypothesis also inspired a range of ML methods, particularly in dimensionality reduction \cite{belkin2003laplacian,coifman2006diffusion,roweis2000nonlinear,coifman2005geometric,balasubramanian2002isomap,elhamifar2013sparse} and manifold-based regularization for semi-supervised and active learning \cite{belkin2006convergence,calder12520novel,niyogi2013manifold}. Unlike classical methods for dimensionality reduction, here we learn variational data embeddings which empirical evidence shows are smoother and more structured than their deterministic counterparts. Unlike learning methods using manifolds strictly for regularization, we leverage the manifold as the support of the data to exploit its geometric information using GNNs and study the impact of doing so on statistical generalization.

%% file: Sections/background.tex
\section{Background} \label{sec:background}

\subsection{Manifold hypothesis and geometric graph approximation}\label{sec:geom_graph_approx}

Under the manifold hypothesis, high-dimensional data can be represented as points $u$ of a $d$-dimensional submanifold $\ccalM$ embedded in some Euclidean space $\reals^F$, i.e., $u \in \ccalM$ and $\ccalM \subset \reals^F$. Since $\ccalM$ is an embedded submanifold, $u$ can be expressed in ambient space coordinates via the map $\ccalX: \ccalM \to \reals^F$. 

In general, we do not know the manifold $\ccalM$, so in order to estimate it we first obtain embeddings $x \in \reals^F$ from the data --- through some dimensionality reduction (e.g., PCA) or learning (e.g., self-supervised) technique --- and assume $x=\ccalX(u)$. The manifold can then be approximated using a geometric graph \cite{netto2025impmnn}. Explicitly, let $x_i$ and $x_j$ denote the embeddings associated with samples $i$ and $j$. These samples are seen as nodes of an undirected graph $G$, and they are connected by an edge with weight
\begin{equation} \label{eqn:geom_graph}
w_{ij} =
\begin{cases}
    \exp{-\frac{\|x_i-x_j\|^2}{2\sigma^2}} \ \mbox{ if } i\neq j \\
     0 \quad \quad \quad \quad \quad \quad \mbox{ if } i=j \text{.}
\end{cases}
\end{equation}
Given $n$ samples, we write the graph adjacency matrix $A_n \in \reals^{n \times n}$ entry-wise as $[A_n]_{ij} = w_{ij}$, and the graph Laplacian as $L_n = \mbox{diag}(A_n\boldsymbol{1})-A_n$. As we discuss in the following $L_n$ provides arbitrarily good approximations of the Laplace-Beltrami (LB) operator of $\ccalM$ as $n \to \infty$. 

\subsection{The Laplace-Beltrami operator and graph Laplacian convergence} \label{sec:lb_operator}

Submanifolds of Euclidean space are locally Euclidean, meaning that in a neighborhood of any given point $u \in \ccalM$, the manifold can be approximated by an Euclidean space via its tangent space.

The tangent space of $\ccalM$ at a point $u \in \ccalM$ consists of all tangent vectors at $u$. A vector $v \in \reals^F$ is considered a tangent vector of $\ccalM$ at $u$ if there exists a smooth curve $\gamma$ such that $\gamma(0) = u$ and $\dot{\gamma}(0)=v$. That is, tangent vectors correspond to the derivatives of curves $\gamma: \reals \to \ccalM$. The tangent space at $u$, denoted $T_u\ccalM$, is therefore defined as $T_u\ccalM = \{\dot{\gamma}(0)\ |\ \mbox{smooth }\gamma: \reals \to \ccalM \mbox{ , } \gamma(0)=u\}$ \cite{robbin2022introduction}. The union of all tangent spaces across the manifold $\ccalM$ forms the tangent bundle $T\ccalM$.

With this notion of tangent space, we can define gradients of functions defined on $\ccalM$. Consider for instance the map $\ccalX: \ccalM \to \reals^F$, which satisfies $\ccalX \in C^\infty(\ccalM)$. The gradient $\nabla \ccalX \in T\ccalM$ is a vector field satisfying $\smash{\langle \nabla \ccalX(u), v \rangle = \frac{\D}{\D t}\bigr|_{t=0} (\ccalX \circ \gamma)(t)}$ for any tangent vector $v \in T_u \ccalM$ and any smooth curve $\gamma$ such that $\gamma(0)=u$ and $\dot{\gamma}(0)=v$ \cite{petersen2006riemannian}. In the opposite direction, given a smooth vector field $V \in T\ccalM$ and an orthonormal basis $e_1, \ldots, e_D$ of $T_u\ccalM$, the divergence $\mathbf{\nabla} \cdot V \in \ccalC^\infty (\ccalM)$ is defined as $\mathbf{\nabla} \cdot V = \sum_{i=1}^D \langle \partial_i V, e_i\rangle$.

By composing the gradient and divergence operators, we obtain the Laplace-Beltrami (LB) operator $\ccalL: \ccalC^\infty (\ccalM) \to \ccalC^\infty (\ccalM)$, given by \cite{berard2006spectral}
\begin{equation}\label{eqn:lb_operator}
\ccalL \ccalX = - \mathbf{\nabla} \cdot \left( \nabla \ccalX \right) \text{.}
\end{equation}
When $\ccalM$ is compact, the operator $\ccalL$ has a discrete, real, and positive spectrum, with eigenvalues $\lambda_i$ and eigenfunctions $\phi_i$, $i=1, 2, \ldots$ (arranged in increasing order of eigenvalues w.l.o.g.).

\noindent \textbf{Convergence of $L_n$ to $\ccalL$.} To relate $L_n$ with the Laplace-Beltrami operator $\ccalL$ of $\ccalM$, one can define the continuous extension $\ccalL_n$ of $L_n$ operating on $\ccalX \in C^\infty(\ccalM)$ as \cite{belkin2008towards}
\begin{equation}
\ccalL_n \ccalX(u) = \ccalX(u)\frac{1}{n}\sum_{i=1}^n e^{-\frac{\|u-u_i\|^2}{2\sigma_n^2}} - \frac{1}{n}\sum_{i=1}^n \ccalX(u_i) e^{-\frac{\|u-u_i\|^2}{2\sigma_n^2}} \text{.}
\end{equation}
By carefully choosing parameters $\{\sigma_n\}$, it can be shown that, for $\ccalX\in C^\infty(\ccalM)$, 
\begin{equation}
{\lim_{n \to \infty} \frac{1}{\sigma_n^{2m + 2}} \ccalL_n \ccalX(u)} = C_\ccalM \ccalL \ccalX(u)
\end{equation}
where $C_\ccalM$ is a constant independent of $n$. Explicitly, the Laplacian of geometric graphs constructed from embeddings $x_i$ as in \eqref{eqn:geom_graph} (which are equal to $\ccalX(u_i)$) converges point-wise to the LB operator of the underlying manifold.

\subsection{Graph semi-supervised learning} \label{sec:ss_learning}

Let $G=(\ccalV,\ccalE)$, $|\ccalV|=n$, be a graph with vertex set $\ccalV$ and edge set $\ccalE \subseteq \ccalV \times \ccalV$. Let $X \in \reals^{n\times F}$ be node attributes or features associated with the nodes of $G$; i.e., each node $i \in \ccalV$ is associated with a $F$-dimensional signal. Suppose we want to use the information in $X$ to assign each node to one of $C$ classes represented by a label vector $y \in \{1,\ldots,C\}^n$. 

The graph semi-supervised approach to this task consists of sampling a training node subset $\ccalT \subset \ccalV$ and solving the following optimization problem:
\begin{align}\label{eqn:erm_ss_learning}
   \min_{h \in \ccalH} R_\ccalT (h) &= \min_{h \in \ccalH} l(y,h(X,G);\ccalT) \nonumber\\
    &:= \min_{h \in \ccalH} \tilde{l}(M_\ccalT y,M_\ccalT h(X,G))
\end{align}
where $\ccalH$ is a hypothesis class, $\tilde{l}$ is a loss function (e.g., the $2$-norm), and $M_{\ccalT} \in \{0,1\}^{|\ccalT|\times n}$ is a matrix acting as the training mask, i.e., each row has \textit{exactly} one non-zero entry, and each column has \textit{at most} one non-zero entry. We call $l$ the semi-supervised loss. Note that, though the loss is only calculated at nodes $i \in \ccalT$, the signal information $X$ across all the nodes in $G$ is used to compute $h(X,G)$. 

Ultimately, we want $h$ to generalize well to the unseen nodes $\ccalV \setminus \ccalT$. This ability is measured by the generalization gap
\begin{equation}\label{eqn:ga}
    GA(h) = |R_{\ccalV \setminus \ccalT}(h) - R_\ccalT(h)| \text{.}
\end{equation}

\subsection{Graph neural networks (GNNs) and GNN convergence} \label{sec:gnn_conv}

GNNs are neural network (NN) architectures tailored to graphs. They have multiple layers, each consisting of a linear map followed by a nonlinear activation function, and each operation is adapted to respect the sparsity pattern of the graph. In practice, this restriction is met by parametrizing the linear map of the NN layer by a graph matrix representation, typically the adjacency matrix or Laplacian. Here, we consider the graph Laplacian $L \in \reals^{n \times n}$. The $\ell$th GNN layer is defined as \cite{ruiz2020gnns}
\vspace{-0.1cm}
\begin{align} \label{eqn:gnn_layer}
    X_{\ell} = \rho(h(X_{\ell-1}, L)) = \rho \bigg( \sum_{k=0}^{K-1} L^k X_{\ell-1} W_{\ell k} \bigg) 
\end{align}
where $X_{\ell} \in \reals^{n \times F_{\ell}}$ and $X_{\ell-1} \in \reals^{n \times F_{\ell-1}}$ are the \textit{embeddings} at layers $\ell$ and $\ell-1$, and $W_{\ell k} \in \reals^{F_{\ell-1} \times F_\ell}$ are learnable parameters. The function $\rho: \reals \to \reals$ is a nonlinear function such as the ReLU or sigmoid, which acts independently on each entry as $[\rho(X)]_{ij} = \rho([X]_{ij})$. 

For an $\mathscr{L}$-layer GNN, the GNN output is $Y = X_{\mathscr{L}}$ and, given input data $X$, $X_0 = X$. For a more compact description, we will represent this GNN as a map $Y = \Phi_\ccalW(X,L)$ parametrized by the learnable weights $\ccalW=\{W_{\ell k}\}_{\ell,k}$ at all layers. 

\noindent \textbf{Convergence to MNNs.} A manifold neural network (MNN) layer is defined pointwise at $u \in \ccalM$ as \cite{wang2021manifold,wang2024manifold}
\begin{equation}
    \ccalX_\ell(u) = \rho \left( \sum_{k=0}^{K-1} \big(e^{-k\ccalL} \ccalX_{\ell-1}\big)(u)W_{\ell k}\right) 
\end{equation}
with $\ccalX_{\ell}: \ccalM \to \reals^{F_{\ell}}$, $W_{\ell k} \in \reals^{F_{\ell-1} \times F_\ell}$, and $\rho$ nonlinear and entry-wise. Once again for compactness, given input $\ccalX_0 = \ccalX$ we represent the whole MNN as a map $\ccalY=\Phi_\ccalW(\ccalX,\ccalL)$.

The following result motivates seeing point classification on manifolds as graph semi-supervised learning, and is the cornerstone of the theoretical generalization results in the next section.

\begin{proposition}[\cite{wang2024manifold}, simplified] \label{prop:transf}
Let $\Phi_\ccalW$ be an MNN on the $d$-dimensional manifold $\ccalM$.
Let $\{u_1,\ldots,u_n\}$ be a set of points sampled uniformly from $\ccalM$ and $L_n$ the corresponding geometric graph Laplacian. Define the map $\ccalP_n: \ccalX \mapsto X_n$:
\begin{equation}
    [X_n]_{ij}=[(\ccalP_n\ccalX)(u_i)]_j = [\ccalX(u_i)]_j \text{.}
\end{equation} 
 Suppose Assumptions \ref{ass:lipschitz_filter_main_body}--\ref{ass:lipschitz_act_fn_main_body} (stated in Section \ref{sec:main_matter}) hold.
Then, with probability at least $1-\delta$,
\begin{equation}
\|\Phi_\ccalW(\ccalP_n\ccalX,L_n)-\ccalP_n\Phi_\ccalW(\ccalX,\ccalL)\| =\ccalO \bigg(\sqrt[d+4]{\frac{\log{1/\delta}}{n}} \bigg)\text{.}
\end{equation}
\end{proposition}

In words, on geometric graphs sampled from a manifold, a GNN with weights $\ccalW$ converges to an MNN with the same set of weights. This constitutes a transferability result, justifying training GNNs on graphs sampled from the manifold, freezing the learned parameters, and applying the model either to the manifold or to additional graphs sampled from the same manifold.

%% file: Sections/method.tex
\section{Classification as graph semi-supervised learning} \label{sec:main_matter}

\begin{figure*}[t]
    \centering
    \begin{subfigure}[t]{0.36\textwidth}
        \centering
        \raisebox{0.1\height}{%
            \includegraphics[width=1.05\textwidth]{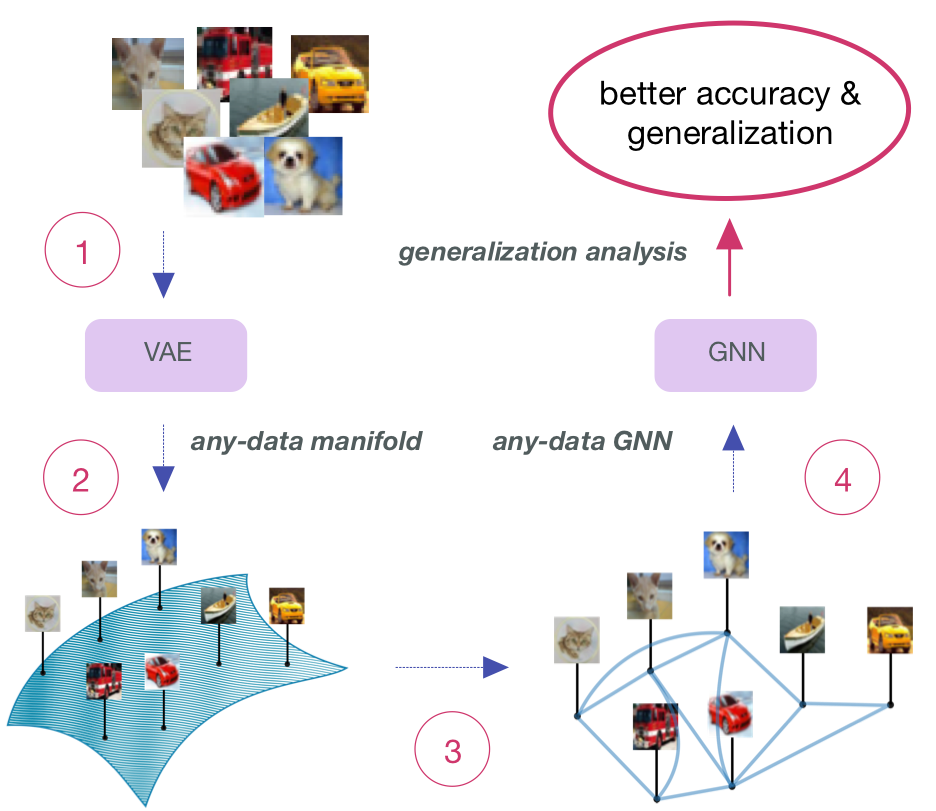}
        }
        \caption{}
        \label{fig:image_manifold}
    \end{subfigure}%
    \hspace{7em}
    \begin{subfigure}[t]{0.37\textwidth}
        \centering
        \begin{setup}
        \setlength{\leftmargin}{5pt}
        \begin{itemize}[leftmargin=*]
            \item The manifold $\ccalM$ is a $d$-dimensional embedded submanifold of $\reals^F$.
            \item We randomly sample $n$ points from $\ccalM$, forming the $n$-node geometric graph $G_n$ with Laplacian $L_n$.
            \item The loss $\tilde{l}$ in \eqref{eqn:erm_ss_learning} is the mean $\ell_2$ norm.
            \item Data $X_n \in \reals^{n \times F}$ is defined as in \eqref{eqn:signal_defn}, with rows $[X_n]_{i:} \in \reals^F$.
            \item GNN sees all of $X_n$ at training, but the loss is computed only on a training subset $\ccalT$ of size $p$.
            \item The test set is $\{1,\ldots,n\} \setminus \ccalT$, with size $q$, so $p+q=n$.
        \end{itemize}
        \end{setup}
        \caption{}
        \label{fig:setup}
    \end{subfigure}
    \caption{\small{(a) Framework schematic. We start by constructing VAE embeddings (1), computing their pairwise distances to form manifolds (2), and sampling graphs from the manifolds (3). GNNs are trained on these graphs to leverage geometric information for image classification (4). (b) Setup for Theorems \ref{thm:ga_bound_1}--\ref{thm:MNN_learning} and Corollary \ref{thm:main}.}}
\end{figure*}

Consider a standard classification task in which the goal is to assign data $\mathrm{X} \in \ccalS_\mathrm{X}$ (the sample space $\ccalS_\mathrm{X}$ is arbitrary) to one of $C$ classes using labels $\mathrm{y} \in \{1,\ldots,C\}$. Given labeled data $\{\mathrm{X_{m}},\mathrm{y}_{m}\}_{m=1}^M$, the classical supervised learning approach consists of selecting a training set $\ccalT \subset \{1,\ldots,M\}$; minimizing some loss over $\ccalT$; and computing the classification accuracy on the test set $\{1,\ldots,M\} \setminus \ccalT$ to evaluate the ability of the model to generalize.

Leveraging the manifold hypothesis, this problem can be parametrized in a different way. The data $\mathrm{X}_m$ are high-dimensional feature vectors, but under the manifold hypothesis, they admit lower-dimensional representations as points $u_m \in \ccalM$ with $\ccalM$ a $d$-dimensional embedded submanifold of $\reals^F$. Suppose we know the map $\psi: \ccalS_{\mathrm{X}} \to \ccalM$ that achieves such lower-dimensional representations, and also the map $\ccalX: \ccalM \to \reals^F$ allowing to write $u \in \ccalM$ in ambient space coordinates as $\ccalX(u) \in \reals^F$. Then we can represent $\mathrm{X_m} \in \ccalS_\mathrm{X}$ as $x_m = \ccalX(\psi(\mathrm{X_m})) \in \reals^F$.

As discussed in Section \ref{sec:geom_graph_approx}, the embeddings $x_m$, when learned, can be used to approximate the manifold $\ccalM$ via a geometric graph $G$ where each sample $m$ is a node and each edge has weight $w_{mm'} = \exp{(-\|x_m-x_{m'}\|^2/2\sigma^2)}$ for $m\neq m'$ [cf. \eqref{eqn:geom_graph}]. Here, we will instead see the graph $G$ as the support of the graph semi-supervised learning problem from Section \ref{sec:ss_learning} parametrized by a GNN.

Specifically, on the graph $G$ define the node attribute matrix $X \in \reals^{n \times F}$ where 
\begin{equation} \label{eqn:signal_defn}
[X]_{i:} = x_i
\end{equation} 
i.e., row $i$ stores the embedding vector corresponding to node $i$. Define also the label vector $y \in \{1,\ldots,C\}^n$ where $[y]_m = \mathrm{y}_m$. The goal is to solve the minimization problem in \eqref{eqn:erm_ss_learning} over hypothesis class 
$\ccalH = \{\Phi_\ccalW(X,L) \text{ s.t. } \ccalW = \{W_{\ell k}\}_{\ell,k}, W_{\ell k} \in \reals^{F_{\ell-1}\times F_\ell}\}$
where $\Phi_\ccalW$ is the GNN composed by layers \eqref{eqn:gnn_layer} and $L$ is the graph Laplacian.

\subsection{Generalization}

The rationale for reformulating standard point classification as semi-supervised learning on a graph is to exploit the geometry in the data to improve predictive performance, as supported by Prop.~\ref{prop:transf}. We first demonstrate this theoretically by showing that the generalization gap of graph semi-supervised learning on geometric graphs sampled from a manifold decreases asymptotically with the graph size.

Before stating our results, we first define our setup in Figure \ref{fig:setup} and state the following assumptions and lemmas.

\begin{assumption} \label{ass:lipschitz_filter_main_body}
The convolutional maps in $\Phi_\ccalW$ are locally Lipschitz continuous on $\ccalM$ and have norm at most 1. 
\end{assumption}

\begin{assumption}\label{ass:bandlimited_main_body}
    The convolutions in all layers of $\Phi_\ccalW$ are low-pass filters with bandwidth $c$, i.e., if $\ccalY$ is the output of a convolution, $\langle \ccalY, \phi_i\rangle = 0$ for $\lambda_i > c$, and $i_c = \argmin_i (\lambda_i - c)\boldsymbol{1}(\lambda_i  \geq c)$.
\end{assumption}


\begin{assumption}\label{ass:lipschitz_act_fn_main_body}
The nonlinear function $\rho$ and its first-order derivative $\rho'$ have Lipschitz constant $1$ and $\rho(0)=0$, i.e., the function is normalized Lipschitz continuous.
\end{assumption}

\begin{theorem}[An unsatisfactory generalization bound]\label{thm:ga_bound_1}
Under \textbf{Setup}, suppose the minimum of the optimization problem in \eqref{eqn:erm_ss_learning} is achieved by $\Phi_{\ccalW_G^*}$, i.e., by the GNN with weights $\ccalW_G^*$, and that Assumptions \ref{ass:lipschitz_filter_main_body}--\ref{ass:lipschitz_act_fn_main_body} hold. Let $p > q$. With probability at least $1-\delta$,
\begin{equation} \label{eqn:incomplete_ga_bd}
\small
GA(\Phi_{\ccalW_G^*}) = \ccalO \bigg(\frac{1}{i_c} + \sqrt[d+4]{\frac{\log{1/\delta}}{n}} 
+ \frac{p-q}{pq}\tilde{l}(\ccalY,\Phi_{\ccalW_{G}^*}(\ccalX,\ccalL))\bigg) \text{.}
\normalsize
\end{equation}
\end{theorem}

\begin{proof}
    Let $R_{\ccalT}(\ccalW_{G}^*) = \tfrac{1}{p}\tilde{l}(M_\ccalT Y_n, M_\ccalT \Phi_{\ccalW_{G}^*}(X_n, L_n))$ and $R_{\ccalV \setminus \ccalT}(\ccalW_{G}^*) = \tfrac{1}{q}\tilde{l}(M_{\ccalV \setminus \ccalT} Y_n, M_{\ccalV \setminus \ccalT}\Phi_{\ccalW_{G}^*}(X_n, L_n))$ be the training and test error, respectively.
    Taking the $L_2$ loss as our loss function, we have that
    \[
    \begin{aligned}
        &R_{\ccalT}(\ccalW_{G}^*) = \dfrac{1}{p}\|M_\ccalT\Phi_{\ccalW_{G}^*}(X_n, L_n) - M_\ccalT Y_n\|_2 \nonumber \\
        &\hspace{1em} = \dfrac{1}{p}\left[\sum_{i\in \ccalT} (\Phi_{\ccalW_{G}^*}(X_n, L_n))_i - \Phi_{\ccalW_{G}^*}(\ccalX,\ccalL)(x_i))^2\right]^{1/2},\\
        &R_{\ccalV \setminus \ccalT}(\ccalW_{G}^*) = \dfrac{1}{q}\|M_{\ccalV \setminus \ccalT}\Phi_{\ccalW_{G}^*}(X_n, L_n) - M_{\ccalV \setminus \ccalT} Y_n\|_2 \nonumber \\
        &\hspace{1em} = \dfrac{1}{q}\left[\sum_{i\in \ccalV \setminus \ccalT} (\Phi_{\ccalW_{G}^*}(X_n, L_n))_i - \Phi_{\ccalW_{G}^*}(\ccalX,\ccalL)(x_i))^2\right]^{1/2}.
    \end{aligned}
    \]
    Under the transductive learning setting, the generalization gap $GA(\Phi_{\ccalW_{G}^*}) = \left|R_{\ccalV \setminus \ccalT}(\ccalW_{G}^*) - R_{\ccalT}(\ccalW_{G}^*)\right|$ is bounded as follows
    \[
    \begin{aligned}
        &GA(\Phi_{\ccalW_{G}^*}) \nonumber \\
        &= \biggl|\dfrac{1}{q}\tilde{l}(M_{\ccalV \setminus \ccalT} Y_n, M_{\ccalV \setminus \ccalT}\Phi_{\ccalW_{G}^*}(X_n, L_n))\nonumber \\
        &\quad - \dfrac{1}{p}\tilde{l}(M_\ccalT Y_n, M_\ccalT \Phi_{\ccalW_{G}^*}(X_n, L_n)\biggr| 
        \underset{(\pm) \tfrac{(p + q)}{pq}\tilde{l}(\ccalY,\Phi_{\ccalW_{G}^*}(\ccalX,\ccalL))}{\leq} \nonumber \\
        &\biggl|\left(\dfrac{1}{q}\tilde{l}(M_{\ccalV \setminus \ccalT} Y_n, M_{\ccalV \setminus \ccalT}\Phi_{\ccalW_{G}^*}(X_n, L_n)) - \dfrac{1}{q}\tilde{l}(\ccalY,\Phi_{\ccalW_{G}^*}(\ccalX,\ccalL))\right) \nonumber\\
            &\quad+\left(\dfrac{1}{p}\tilde{l}(\ccalY,\Phi_{\ccalW_{G}^*}(\ccalX,\ccalL)) - \dfrac{1}{p}\tilde{l}(M_\ccalT Y_n, M_\ccalT \Phi_{\ccalW_{G}^*}(X_n, L_n)\right)\nonumber\\
            &\quad+\left(\dfrac{(p-q)}{pq}\tilde{l}(\ccalY,\Phi_{\ccalW_{G}^*}(\ccalX,\ccalL))\right)\biggr|\\
            \nonumber\\
            &\leq \underbrace{\dfrac{1}{q}\left|\tilde{l}(M_{\ccalV \setminus \ccalT} Y_n, M_{\ccalV \setminus \ccalT}\Phi_{\ccalW_{G}^*}(X_n, L_n)) - \tilde{l}(\ccalY,\Phi_{\ccalW_{G}^*}(\ccalX,\ccalL))\right|}_{\circled{2}}\nonumber\\
            &\quad+\underbrace{\dfrac{1}{p}\left|\tilde{l}(\ccalY,\Phi_{\ccalW_{G}^*}(\ccalX,\ccalL)) -\tilde{l}(M_\ccalT Y_n, M_\ccalT \Phi_{\ccalW_{G}^*}(X_n, L_n)\right|}_{\circled{3}}\nonumber\\
            &\quad+\tfrac{(p - q)}{pq}|\tilde{l}(\ccalY,\Phi_{\ccalW_{G}^*}(\ccalX,\ccalL))|.
    \end{aligned}
    \]

    From Lemma~\ref{lem:ga_bound_1}, we have that
    \[
        \begin{aligned}
            |\tilde{l}(\ccalY,\Phi_\ccalW(\ccalX,\ccalL))-\tilde{l}(M_\ccalT Y_n, M_\ccalT \Phi_\ccalW(X_n,L_n))|\nonumber\\
            = \ccalO \bigg(\frac{1}{i_c} + \sqrt[d+4]{\frac{\log{1/\delta}}{n}} \bigg) \nonumber.
        \end{aligned}
    \]
    That completes the proof since the previous lemma provides the bounds for \circled{2} and \circled{3}.
\end{proof}

We observe that the generalization gap is upper-bounded by three terms. The first is a term relating to the convolutional filter bandwidth $c$, which is constant but small for filters with sufficiently high bandwidth. The second arises from the convergence of GNNs on graph sequences converging to a manifold [cf. Prop. \ref{prop:transf}], vanishing asymptotically with $n$. The third depends on the training and test set sizes $p$ and $q$, as well as on the loss achieved by GNN $\Phi_{\ccalW_G^*}$ on the \textit{entire} manifold $\ccalM$. 

The third term is interesting as the factor $(p-q)/pq$ highlights the role of the statistical imbalance, i.e., of the different training and test set sizes, on the generalization gap. To see this, consider the common scenario in which $p$ and $q$ are fixed proportions of $n$, e.g., $p=\nu n$ and $q=(1-\nu)n$. Then, the third term in \eqref{eqn:incomplete_ga_bd} is $\ccalO(n^{-1}\tilde{l}(\ccalY,\Phi_{\ccalW_{G}^*}(\ccalX,\ccalL))$ unless $\nu=0.5$ -- i.e., balanced training and test sets --, in which case this term is zero and does not contribute to the generalization bound.

When $\nu > 0.5$, the extent to which the third term dominates or not the generalization gap depends on the loss realized by $\Phi_{\ccalW_{G}^*}$ on $\ccalM$. This dependence is unsatisfactory for two reasons. Since the loss is computed over all of $\ccalM$, it depends on the test set; and further, since $\Phi_{\ccalW_G^*}$ is optimized for $G_n$ (and not for $\ccalM$), it is not even clear that $\Phi_{\ccalW_G^*}$ minimizes the loss on the manifold. 

Specifically, we can chain Theorem \ref{thm:ga_bound_1} and Lemma \ref{lem:ga_bound_1} once more to derive an upper bound on the generalization gap that no longer depends on the loss on the entire manifold, but rather on the minimum semi-supervised training loss on the graph $G_n$: 
\begin{equation}
    l_G^* = \tilde{l}(M_\ccalT Y_n, M_\ccalT \Phi_{\ccalW_G^*}(X_n,L_n)) \text{.}
\end{equation}

\begin{theorem}[A satisfactory generalization bound]\label{thm:ga_bound_2}
    Under \textbf{Setup}, suppose the minimum of the optimization problem in \eqref{eqn:erm_ss_learning} is achieved by $\Phi_{\ccalW_G^*}$, i.e., by the GNN with weights $\ccalW_G^*$, and that Assumptions \ref{ass:lipschitz_filter_main_body}--\ref{ass:lipschitz_act_fn_main_body} hold. Let $p > q$. With probability at least $1-\delta$, 
    \begin{equation} \label{eqn:better_ga_bd}
    GA(\Phi_{\ccalW_G^*})
    = \ccalO \bigg(\frac{p}{q i_c} + \sqrt[d+4]{\frac{\log{1/\delta}}{n}} + \frac{p-q}{pq}l^*_G\bigg) \text{.}
    \end{equation}
\end{theorem}
\begin{proof}
    \begin{align}
        &\tfrac{(p - q)}{pq}|\tilde{l}(\ccalY,\Phi_{\ccalW_{G}^*}(\ccalX,\ccalL))| \underset{\pm\tilde{l}(M_\ccalT Y_n, M_\ccalT \Phi_{\ccalW_{G}^*}(X_n, L_n))}{=} \nonumber\\
        &\tfrac{(p - q)}{pq}|\tilde{l}(\ccalY,\Phi_{\ccalW_{G}^*}(\ccalX,\ccalL)) - \tilde{l}(M_\ccalT Y_n, M_\ccalT \Phi_{\ccalW_{G}^*}(X_n, L_n)) \nonumber \\
        &\hspace{2em}+ \tilde{l}(M_\ccalT Y_n, M_\ccalT \Phi_{\ccalW_{G}^*}(X_n, L_n))|= \nonumber\\
        &\underbrace{\tfrac{(p-q)}{pq}|\tilde{l}(\ccalY,\Phi_{\ccalW_{G}^*}(\ccalX,\ccalL)) - \tilde{l}(M_\ccalT Y_n, M_\ccalT \Phi_{\ccalW_{G}^*}(X_n, L_n))|}_{\circled{3}, \ \textbf{mult. by factor} \ \tfrac{(p-q)}{pq}} \nonumber \\
        &\hspace{2em}+ \tfrac{(p-q)}{pq}|\tilde{l}(M_\ccalT Y_n, M_\ccalT \Phi_{\ccalW_{G}^*}(X_n, L_n))|.
    \end{align}

    Finally, recapping the definition for the minimum semi-supervised training loss on the graph $G$ as
    \begin{align}
        l_G^* = \tilde{l}(M_\ccalT Y_n, M_\ccalT \Phi_{\ccalW_G^*}(X_n, L_n)),
    \end{align}
    with some additional algebraic manipulation of the constant factors, we achieve the bound. 
\end{proof}

The generalization bound in Theorem \ref{thm:ga_bound_2} is more satisfactory, as now the term depending on the loss realized by the GNN can be controlled through optimization over the training set $\ccalT$. However, this comes at the cost of an increase in the constant term from $1/i_c$ in Theorem \eqref{thm:ga_bound_1} to $p/(q i_c)$ in Theorem \ref{thm:ga_bound_2}. In modern machine learning, one typically has significantly more training samples $p$ than test samples $q$. Hence, this increase might be non-negligible in practice.
\subsection{Learning on graphs of increasing size}
In this section we discuss an alternative GNN training algorithm inspired by \cite{cervino2021increase} allowing to directly minimize $\tilde{l}(\ccalY,\Phi_\ccalW(\ccalX,\ccalL))$, the loss on the manifold, and as such to curb the increase in the generalization gap observed in Theorem \ref{thm:ga_bound_2}.

The algorithm is rather simple. Instead of fixing the graph $G_n$ during the entire training process, we instead start from an $n_0$-node graph $G_{n_0}$ and, after $\Delta t$ gradient updates over this graph, resample a graph $G_{n_1}$ with $n_1 = n_0 + \Delta n$ from $\ccalM$. We proceed to do $\Delta t$ gradient updates over $G_{n_1}$, then resample $G_{n_2}$ and repeat. 
Explicitly, the $k$th iterate is given by
\begin{align}
	\ccalW_{k+1}&=\ccalW_k-\eta_k \nabla_{\ccalW} l(Y_{n_t},\Phi(X_{n_t},L_{n_t})), \label{eqn:GNN_Learning_Step}
\end{align}
with $t= \lfloor k / \Delta t \rfloor$.

Under mild assumptions, it can be shown that the GNN obtained by solving problem \eqref{eqn:erm_ss_learning} on this graph sequence minimizes the empirical risk on the manifold $\ccalM$.

\begin{theorem}\label{thm:MNN_learning}
Under \textbf{Setup}, let $\Phi_{\ccalW}$ be a GNN learned with iterates \eqref{eqn:GNN_Learning_Step}. If at each step $k$ the number of nodes $n_t$ is such that 
\begin{align}
	\mbE[\|\nabla_{\ccalW}\tilde{l}(\ccalY,\Phi_{\ccalW_k}(\ccalX,\ccalL)) -\nabla_{\ccalW} l(Y_{n_t},\Phi_{\ccalW_k}(X_{n_t},L_{n_t}))\|]\nonumber\\
	 < \|\nabla_{\ccalW}\tilde{l}(\ccalY,\Phi_{\ccalW_k}(\ccalX,\ccalL))\| - {\epsilon},
\end{align}
then after at most $k^*=\ccalO(1/\epsilon^2)$ iterations $\Phi_{\ccalW_{G_{n_t}}^*} = \Phi_{\ccalW_{k^*}}$ is within an $\epsilon$-neighborhood of the solution of the empirical risk minimization problem on $\ccalM$. 
\end{theorem}
\begin{proof}
For every $\epsilon>0$, we define the stopping time $k^*$ as
\begin{align}
    k^*:=\min_{k\geq 0}\{\mbE[\|\nabla_{\ccalW} \tilde{l}(Y,\Phi_{\ccalW_k}(\ccalX,\ccalL))\|]\leq  \gamma \varepsilon+\epsilon \}.
\end{align}
Given the final iterates at $k=k^*$ and the initial values at $k=0$, we can express the expected difference between the loss $\tilde{l}$ as the summation over the difference of iterates,
\begin{align}
    &\mbE[\tilde{l}(Y,\Phi_{\ccalW_0}(\ccalX,\ccalL))-\tilde{l}(\ccalY,\Phi_{\ccalW_{k^*}}(\ccalX,\ccalL))] \nonumber\\
    &= \mbE\left[\sum_{k=1}^{k^*}\tilde{l}(\ccalY,\Phi(\ccalX;\ccalH_{k-1},\ccalL))-\tilde{l}(\ccalY,\Phi(\ccalX;\ccalH_{k},\ccalL)\right] \text{.} 
\end{align}
    Taking the expected value with respect to the final iterate $k=k^*$, we get
\begin{align}
    &\mbE\bigg[\tilde{l}(\ccalY,\Phi_{\ccalW_0}(\ccalX,\ccalL))-\tilde{l}(\ccalY,\Phi_{\ccalW_{k^*}}(\ccalX,\ccalL))\bigg] \nonumber\\
    &=\mathop{\mbE}_{k^*}\bigg[\mbE\bigg[\sum_{k=1}^{k^*}\tilde{l}(\ccalY,\Phi_{\ccalW_{k-1}}(\ccalX,\ccalL))-\tilde{l}(\ccalY,\Phi_{\ccalW_k}(\ccalX,\ccalL)\bigg]\bigg] \nonumber\\
    &=\sum_{t=0}^\infty \mbE\bigg[\sum_{k=1}^{t}\tilde{l}(\ccalY,\Phi_{\ccalW_{k-1}}(\ccalX,\ccalL)) \nonumber\\
    &\hspace{4em}-\tilde{l}(\ccalY,\Phi_{\ccalW_k}(\ccalX,\ccalL)\bigg]P(k^*=t) \text{.}\label{eqn:lemma6_law_of_total_probability}
\end{align}
Lemma \ref{lemma:martingale} applied to any $k\leq k^*$ verifies
\begin{align}
    &\mbE\bigg[\tilde{l}(\ccalY,\Phi_{\ccalW_{k-1}}(\ccalX,\ccalL))-\tilde{l}(\ccalY,\Phi_{\ccalW_k}(\ccalX,\ccalL))\bigg] \geq \eta \gamma \epsilon^2 \text{.}
\end{align}
Coming back to \eqref{eqn:lemma6_law_of_total_probability}, we get
\begin{align}
    &\mbE\bigg[\tilde{l}(\ccalY,\Phi(\ccalX;\ccalH_{0},\ccalL))-\tilde{l}(\ccalY,\Phi_{\ccalW_{k^*}}(\ccalX,\ccalL))\bigg] \nonumber\\
    &\geq \eta\gamma \epsilon^2 \sum_{t=0}^\infty t P(k^*=t)= \eta \gamma \epsilon^2 \mbE[k^*] \text{.}
\end{align}
Since the loss function $\tilde{l}$ is non-negative, 
\begin{align}
    \frac{\mbE\big[\tilde{l}(\ccalY,\Phi_{\ccalW_0}(\ccalX,\ccalL))\big]}{ \eta \gamma \epsilon^2 }
    &\geq  \mbE[k^*],
\end{align}
from which we conclude that $k^* =\ccalO(1/\epsilon^2)$.
\end{proof}

This result is of independent interest, as it prescribes an algorithm for achieving approximate solutions of risk minimization problems on manifolds by solving them on sequences of geometric graphs. 
In our specific context, it further allows one to obtain GNNs with improved generalization gap. This is done by combining Theorems \ref{thm:ga_bound_1} and \ref{thm:MNN_learning} in the following corollary.

\begin{corollary}[A better generalization bound]\label{thm:main}
    Let $l^*_\ccalM = \min_\ccalW \tilde{l}(\ccalY,\Phi_\ccalW(\ccalX,\ccalL))$. Under \textbf{Setup}, let $\Phi_{\ccalW_{G_{n_t}}^*}$ be the GNN learned on a sequence of graphs as in Theorem \ref{thm:MNN_learning}. With probability at least $1-\delta$, 
    \begin{equation} \label{eqn:last_ga_bd}
    GA(\Phi_{\ccalW_{G_{n_t}^*}})
    = \ccalO \bigg(\frac{1}{i_c} + \sqrt[d+4]{\frac{\log{1/\delta}}{n}} + \frac{p-q}{pq}(l^*_\ccalM + \epsilon)\bigg) \text{.}
    \end{equation}
\end{corollary}


This theorem provides both a tighter generalization bound and a practical training guarantee. Relative to Theorem \ref{thm:ga_bound_2}, the first term no longer depends on the fraction between training and test set sizes ($p$ and $q$, respectively), which removes the loose scaling and yields a uniformly tighter bound across splits. Relative to Theorem \ref{thm:ga_bound_1}, it is more practical because it pairs the bound with a constructive procedure, i.e., training a GNN on geometric graphs sampled from the underlying manifold and applying the increasing graph-size algorithm. Since optimizing directly on the manifold is intractable, this theorem justifies the graph approximation and certifies that, for suitable graph sizes and sample budgets, the learned GNN attains a generalization gap within an $\epsilon$-neighborhood of the minimum loss on the manifold.













%% file: Sections/experiments.tex
\section{Experiments}

\begin{figure*}[t]
    \centering
    \begin{subfigure}[t]{0.35\textwidth}
        \centering
        \includegraphics[width=\linewidth,height=3.9cm]{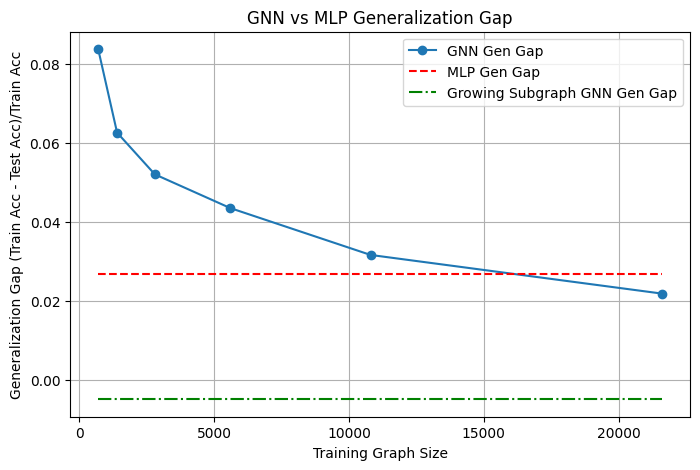}
        \caption{MNIST}
        \label{fig:gen_gap_mnist}
    \end{subfigure}
    \hspace{3.0em}
    \begin{subfigure}[t]{0.35\textwidth}
        \centering
        \includegraphics[width=\linewidth,height=3.9cm]{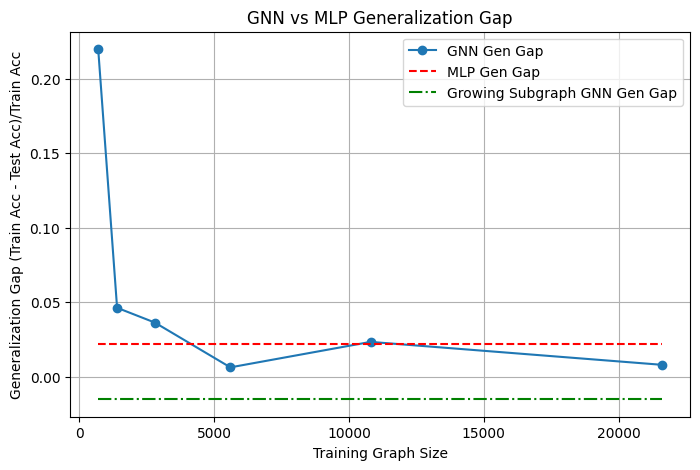}
        \caption{FMNIST}
        \label{fig:gen_gap_fmnist}
    \end{subfigure}
    \hfill
    \begin{subfigure}[t]{0.35\textwidth}
        \centering
        \includegraphics[width=\linewidth,height=3.9cm]{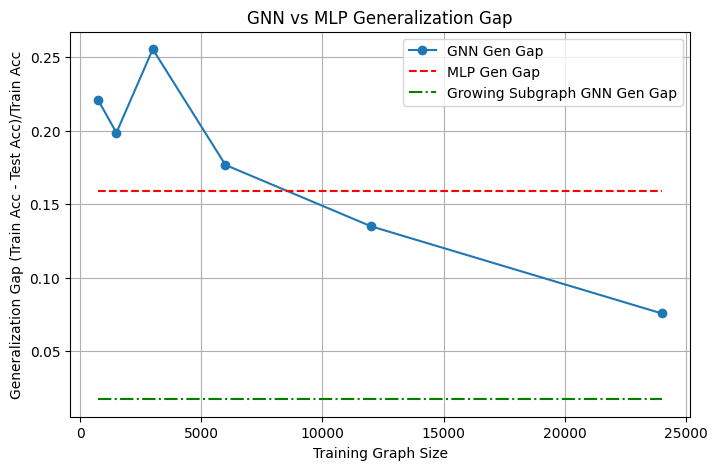}
        \caption{CIFAR10}
        \label{fig:gen_gap_cifar10}
    \end{subfigure}
    \hspace{3.0em}
    \begin{subfigure}[t]{0.35\textwidth}
        \centering
        \includegraphics[width=\linewidth,height=3.9cm]{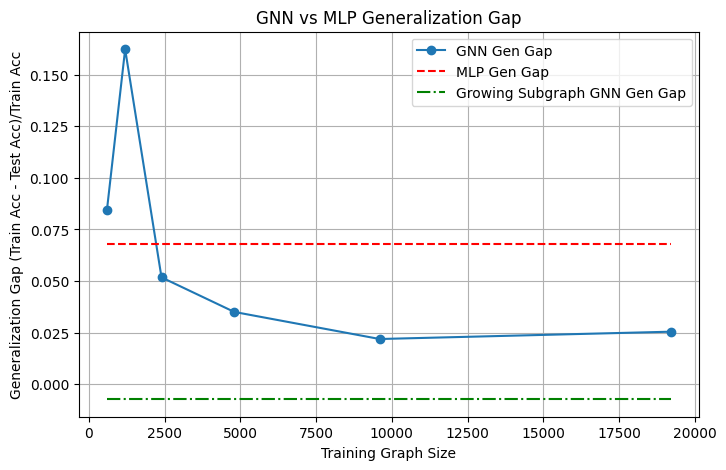}
        \caption{FER2013}
        \label{fig:gen_gap_fer2013}
    \end{subfigure}
    \caption{\small{Generalization gap 
    relative to training accuracy for (a) MNIST, (b) FMNIST, (c) CIFAR10, (d) FER2013. We compare an MLP trained on the VAE embeddings of the full dataset (red); GNNs \textbf{fully} trained on subgraphs of the full data graph with size given by the $x$-axis (blue, Thm.~\ref{thm:ga_bound_2}); and a GNN learned on this sequence of subgraphs, one per epoch (green, Cor.~\ref{thm:main}). The generalization gap decreases with graph size (blue), and is substantially smaller when training on growing subgraphs (green), in line with our theoretical predictions.}}
\label{fig:gen_gap_all}
\end{figure*}

\begin{table*}[t]
    \centering
    \caption{\small{Accuracy on the \textbf{full dataset/graph}. Our method outperforms compared methods on every dataset, achieving the highest test accuracy and smallest generalization gap.}}
    \begin{tabular}{lcccccccc}
        \toprule
        \multirow{2}{*}{Model} 
            & \multicolumn{2}{c}{MNIST} 
            & \multicolumn{2}{c}{FMNIST} 
            & \multicolumn{2}{c}{CIFAR10}
            & \multicolumn{2}{c}{FER2013} \\
        \cmidrule(lr){2-3} \cmidrule(lr){4-5} \cmidrule(lr){6-7} \cmidrule(lr){8-9}
            & Test & Train
            & Test & Train 
            & Test & Train 
            & Test  & Train  \\
        \midrule
        GCN (superpixel graph) \cite{dwivedi2023benchmarking} 
            & 90.12 & 96.46 
            & -- & -- 
            & 54.14 & 70.16 
            & -- & -- \\
        $k$NN 
            & 96.31 & 96.92 
            & 83.76 & 86.40 
            & 40.93 & 43.65 
            & 36.58 & 58.81 \\
        MLP 
            & 97.40 & 100.00 
            & 84.35 & 86.53
            & 54.29 & 66.49
            & 42.40 & 50.05 \\
        $\text{GNN}$ (ours) 
            & \textbf{100.00} & 100.00 
            & \textbf{84.46} & 85.28 
            & \textbf{61.83} & 63.18 
            & \textbf{48.38} & 47.97 \\
        \bottomrule
    \end{tabular}
    \label{tab:accuracy_all}
\end{table*}

\begin{figure*}[t]
    \centering
    \begin{subfigure}[t]{0.3\textwidth}
        \centering
        \includegraphics[width=\linewidth,height=4.3cm]{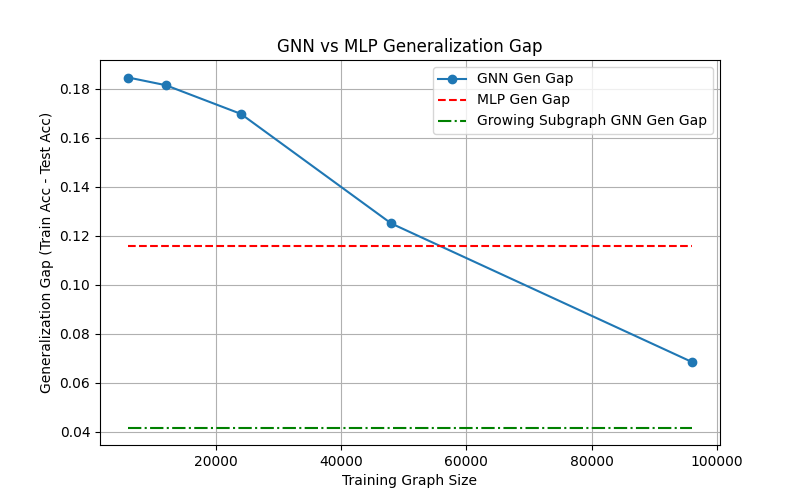}
        \caption{CelebA-Smiling}
        \label{fig:gen_gap_celeba_sb}
    \end{subfigure}
    \begin{subfigure}[t]{0.3\textwidth}
        \centering
        \includegraphics[width=\linewidth,height=4.3cm]{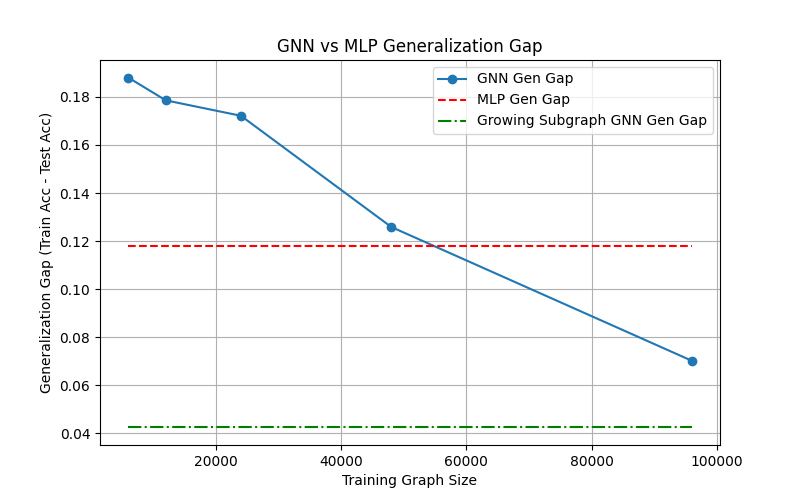}
        \caption{CelebA-Gender}
        \label{fig:gen_gap_celeba_gb}
    \end{subfigure}
    \begin{subfigure}[t]{0.3\textwidth}
        \centering
        \includegraphics[width=\linewidth,height=4.3cm]{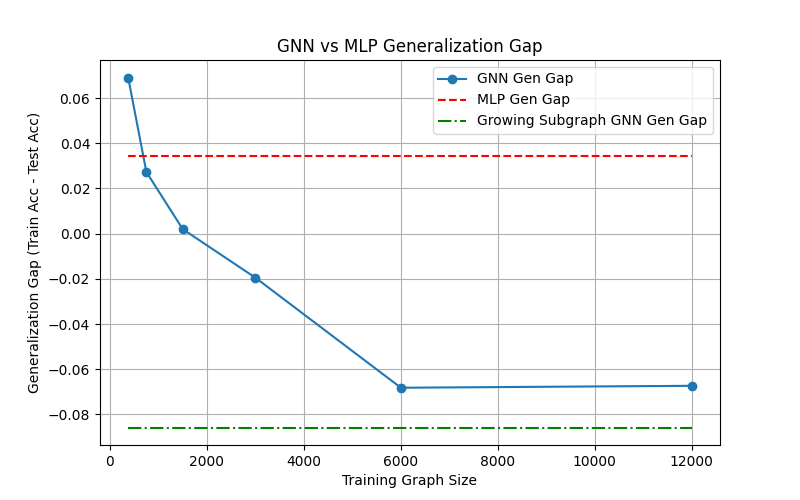}
        \caption{PathMNIST}
        \label{fig:gen_gap_pathmnist}
    \end{subfigure}
    \hfill
    \caption{\small{Generalization gap 
    relative to training accuracy for (a) CelebA-Smiling, (b) CelebA-Gender, (c) PathMNIST. We compare an MLP trained on the VAE embeddings of the full dataset (red); GNNs \textbf{fully} trained on subgraphs of the full data graph with size given by the $x$-axis (blue, Thm.~\ref{thm:ga_bound_2}); and a GNN learned on this sequence of subgraphs, one per epoch (green, Cor.~\ref{thm:main}). The generalization gap decreases with graph size (blue) and is substantially smaller when training on growing subgraphs (green), in line with our theoretical predictions.}}
\label{fig:gen_gap_add_exp}
\end{figure*}

\begin{table*}[t]
    \centering
    \caption{\small{Accuracy on the \textbf{full dataset/graph}. Our method outperforms compared methods on CelebA-Smiling, CelebA-Gender, and PathMNIST, achieving the highest test accuracy and smallest generalization gap. Given the size of the graphs for the first two datasets ($> 162k$ images/nodes), we didn't have time to finish assessing the training accuracy for the $k$NN model.}}
    \begin{tabular}{lcccccc}
        \toprule
        \multirow{2}{*}{Model} 
            & \multicolumn{2}{c}{CelebA-Smiling} & \multicolumn{2}{c}{CelebA-Gender}
            & \multicolumn{2}{c}{PathMNIST}\\
        \cmidrule(lr){2-3} \cmidrule(lr){4-5} \cmidrule(lr){6-7}
            & Test  & Train  
            & Test  & Train
            & Test  & Train\\
        \midrule
        GCN (superpixel graph) \cite{dwivedi2023benchmarking}  
            & -- & -- & -- & -- & -- & --\\
        $k$NN  
            & 70.15 &  (\textit{timeout}) & 79.09 & (\textit{timeout}) & 60.67 & 72.92\\
        MLP 
            & 81.33 & 93.92 & 81.38 & 93.05 & 66.16 & 70.16 \\
        $\text{GNN}$ (ours)  
            & \textbf{87.58} & 90.37 & \textbf{87.51} & 90.32 & \textbf{72.95} & 66.46 \\
        \bottomrule
    \end{tabular}
    \label{tab:accuracy_add_exp}
\end{table*}


\textbf{Experimental setup.}~We conduct experiments on MNIST, FMNIST, CIFAR10, FER2013, CelebA and PathMNIST benchmarks \cite{lecun1998mnist,xiao2017fashion,krizhevsky2009learning, goodfellow2013challenges, liu2015faceattributes, medmnistv1}. 
Since these are image datasets, we first have to define a way to extract meaningful graphs from this setting. A natural approach is to first construct embeddings that represent each image, usually of a lower dimension than the input (image) space, and take advantage of the geometry of such a lower-dimensional manifold with graphs. 

In this work, we make use of autoencoders to build representative embeddings. Since we want to preserve the images' translational invariances/equivariances, we 
set our encoder/decoder networks to be Convolutional Neural Networks (CNNs). In addition, to account for implicit invariances/equivariances in the data, which might not be captured by explicit symmetries incorporated in the model's architecture, we propose to use Variational Autoencoders (VAEs) \cite{kingma2014auto} to learn the latent space. Since VAEs learn a Gaussian approximation of the embeddings' distribution in the latent space, they add more structure to the low-dimensional manifold, which makes it smoother than deterministic AE counterparts, as seen in previous works \cite{pu2016variational,kusner2017grammar}. 

Given a set of images $\{\mathrm{X}_m\}_{m=1}^M$ from the ambient space $\ccalS_\mathrm{X}$, the encoder $f_{\text{enc}}\colon \ccalS_\mathrm{X} \to \reals^F$ reduces the data to a $F$-dimensional embedding $z_m = f_{\text{enc}}(\mathrm{X}_m)$, while the decoder $f_{\text{dec}}\colon \reals^F \to \ccalS_\mathrm{X}$ maps the embedding back to the original space $\hat{\mathrm{X}}_m = f_{\text{dec}}(z_m)$. In our setting, our embedding is defined as the posterior distribution's estimated mean.
For MNIST, CelebA, and PathMNIST, we found that the best latent space has size $F = 128$, for FMNIST $F = 256$, for CIFAR10 $F = 1024$, and FER2013 $F=64$.

Having access to the embeddings $z_m$, we can approximate the image manifold with a graph by computing the pairwise distance between image embeddings following the steps from Section \ref{sec:geom_graph_approx} (Eq.~\ref{eqn:geom_graph}), and then process this graph using a GNN to predict the image labels via semi-supervised node (image) classification. Concretely, given a dataset consisting of pairs $\{z_m, \mathrm{y}_m\}_{m = 1}^M$, where $\mathrm{y}_m \in \{1, \dots, C\}$ 
is the class label for image $m$, we construct a graph $G$ by considering the image embeddings ($z_m$) to be nodes and computing their pairwise edge weights with a Gaussian kernel. However, since computing pairwise distances between all embeddings in the dataset would be impractical, in practice, we use an approximate nearest neighbor (ANN) algorithm to construct a 100-nearest neighbor graph. Specifically, we apply a tree-based ANN method \cite{annoy_git} to find neighbors efficiently and then assign edges with Gaussian weights $\smash{w_{ij} = \exp(-\frac{|z_i - z_j|^2}{2\sigma^2})}$.


\textbf{Experimental results.}~We present our empirical results under three perspectives: (i) adherence to the theoretical results, (ii) effectiveness of our model, measured in terms of image classification accuracy on the test set of standard splits, and (iii) flexibility of our method. It is worth noting that all experiment details are provided in Appendix \ref{sec:exp_details}.

For (i), as shown in Figure~\ref{fig:gen_gap_all}, GNNs trained on fixed subgraphs (blue) exhibit large generalization gaps for small training graph sizes, but the gap decreases steadily with more nodes, eventually outperforming the MLP baseline. This behavior is consistent with the prediction of Theorem~\ref{thm:ga_bound_2}. GNNs trained on sequences of growing subgraphs (green) achieve the smallest generalization gap across all datasets, in agreement with Corollary~\ref{thm:main}, and consistently outperform both fixed-graph GNNs and MLPs.

For (ii), as shown in Table~\ref{tab:accuracy_all} our GNN achieves the highest test accuracy across all four datasets when trained on the full data graph. On MNIST, it reaches perfect accuracy, as expected given the simplicity of the task. On FMNIST and FER2013, our model outperforms all compared methods by a notable margin. While the MLP performs slightly better than $k$NN on CIFAR10, our GNN method surpasses both, achieving the best accuracy with substantially reduced overfitting as reflected by the smaller gap between train and test performance.




Finally, for (iii), we included two sets of experiments to showcase the flexibility of our framework to a broad range of image datasets, which also complements the support on both theory and practical effectiveness. In the first set, we applied it to two large-scale datasets outside standard benchmarks: (i) CelebA \cite{liu2015faceattributes}, a diverse dataset of celebrities' faces in different poses, backgrounds, and attributes, and (ii) PathMNIST \cite{medmnistv1}, a dataset for histopathology detection -- medical image analysis. It is worth noting that, since CelebA has multiple labels for each image, we selected two attributes, i.e., smiling and gender, and framed the tasks as separate binary classifications. The resulting datasets were coined CelebA-Smiling and CelebA-Gender.

The second set of experiments aimed to assess the superiority of the embeddings captured by our proposed CNNVAE model. To this end, we generated alternative embeddings using PCA across all datasets, and then trained and evaluated a GNN using these representations.

As seen in Figure \ref{fig:gen_gap_add_exp} and Table \ref{tab:accuracy_add_exp}, our method still outperforms the baselines throughout all the new datasets, though the graph is larger than the ones we previously tested (more than twice the size.) These results not only align with our theoretical claims -- i.e., strong performance and a shrinking generalization gap as graph size increases -- but also highlight the flexibility of our approach. Furthermore, when comparing GNNs trained on VAE versus PCA-based embeddings (Table \ref{tab:accuracy_vae_vs_pca}), our method maintained superior performance across all datasets. 

\begin{table*}[t]
    \centering
    \caption{\small{Accuracy on the \textbf{full dataset/graph}. VAE-based embeddings provided a more structured latent space, which translates to our method outperforming one that was trained on embeddings generated using PCA. Across all datasets -- MNIST, FMNIST, CIFAR10, FER2013, CelebA-Smiling, CelebA-Gender, and PathMNIST -- a GNN trained on the VAE embeddings achieved the highest test accuracy.}}
    \begin{tabular}{lcccccccccccccc}
        \toprule
        \multirow{2}{*}{Model} 
            & \multicolumn{2}{c}{MNIST}
            & \multicolumn{2}{c}{FMNIST}
            & \multicolumn{2}{c}{CIFAR10}
            & \multicolumn{2}{c}{FER2013}
            & \multicolumn{2}{c}{CelebA-Smiling} & \multicolumn{2}{c}{CelebA-Gender}
            & \multicolumn{2}{c}{PathMNIST}\\
        \cmidrule(lr){2-3} \cmidrule(lr){4-5} \cmidrule(lr){6-7} \cmidrule(lr){8-9} \cmidrule(lr){10-11} \cmidrule(lr){12-13}
        \cmidrule(lr){14-15}
            & Test  & Train  
            & Test  & Train
            & Test  & Train
            & Test  & Train
            & Test  & Train
            & Test  & Train
            & Test  & Train\\   
        \midrule
        $\text{GNN}$ (PCA)   
            & 55.10 & 53.36 & 66.23 & 66.14 & 27.25 & 26.82 & 27.40 & 26.94 & 60.72 & 61.99 & 74.01 & 73.52 & 72.24 & 65.20\\
        $\text{GNN}$ (VAE)  
            & \textbf{100.00} & 100.00 & \textbf{84.46} & 85.25 & \textbf{61.83} & 63.18 & \textbf{48.38} & 47.97 & \textbf{87.58} & 90.37 & \textbf{87.51} & 90.32 & \textbf{72.95} & 66.46\\
        \bottomrule
    \end{tabular}
    \label{tab:accuracy_vae_vs_pca}
\end{table*}



%% file: Sections/discussions.tex
\section{Conclusions}
We introduced an image classification framework that constructs a geometric graph from VAE embeddings and leverages GNNs for semi-supervised learning. Grounded in the manifold hypothesis, this approach treats embeddings as signals over a geometric graph, providing a basis for analyzing GNN generalization.
We show that the generalization gap decreases with more sampled nodes, as also validated empirically. Our model also achieves competitive accuracy, outperforming others on all datasets.
Identified limitations are that our framework depends on the quality of VAE embeddings, which may not always capture meaningful manifold structure. The graph construction also relies on distance metrics and kernel parameters that can impact performance. Additionally, the two-stage pipeline introduces computational overhead that may limit scalability to larger datasets. 

%% file: Sections/appendix.tex
\appendices


\begin{table*}[t]
    \centering
    \caption{\small{VAE training hyperparameters for each dataset.}}
    \begin{tabular}{lcccccc}
    \toprule
    \textbf{Parameter} & \textbf{MNIST} & \textbf{FMNIST} & \textbf{CIFAR10} & \textbf{FER2013} & \textbf{CelebA} & \textbf{PathMNIST} \\
    \midrule
    Batch size         & 64    & 64    & 64   & 64   & 64 & 64\\
    Learning rate      & 0.0001 & 0.0001 & 0.0001 & 0.0003 & 0.0003 & 0.0001\\
    Number of epochs   & 50 & 50 & 50 & 50 & 50 & 50 \\
    Latent size & 128 & 256 & 1024 & 64 & 128 & 128\\
    Num. of layers & [3, 3, 1] & [3, 3, 1] & [4, 5, 2] & [3, 3, 3] & [3, 3, 3] & [3, 3, 1] \\
    \bottomrule
    \end{tabular}
    \label{tab:vae_hyperparams}
\end{table*}


\begin{table*}[t]
    \centering
    \caption{\small{GNN training hyperparameters for each dataset.}}
    \begin{tabular}{lcccccc}
    \toprule
    \textbf{Parameter} & \textbf{MNIST} & \textbf{FMNIST} & \textbf{CIFAR10} & \textbf{FER2013} & \textbf{CelebA} & \textbf{PathMNIST} \\
    \midrule
    Batch size   & 256   & 256   & 256   & 256   & 256 & 256\\
    Learning rate   & 0.01     & 0.01     & 0.01     & 0.01 & 0.01 & 0.01\\
    Kernel width   & 4.0   & 0.8   & 5.0   & 4.0   & 3.5 & 5.0\\
    Hidden dimension & 128 & 128 & 128 & 128 & 128 & 128\\
    Num. of layers & 1 & 1 & 1 & 1 & 1 & 1\\
    \bottomrule
    \end{tabular}
    \label{tab:gnn_hyperparams}
\end{table*}

\section{Experiments Details}\label{sec:exp_details}
Experiments were conducted with two different settings depending on the memory complexity related to the size of the data manifold and its dimension. Specifically, for smaller graphs, i.e., MNIST, FMNIST, and FER2013 datasets, we used a server with 2x NVIDIA GeForce RTX 4090 (24GB) GPU, 128GB of RAM, and a CPU 
AMD Ryzen Threadripper PRO 5955WX 16-Cores. For medium-to-large ones, i.e., CIFAR10, CelebA \cite{liu2015faceattributes}, and PathMNIST \cite{medmnistv1}  we experimented using a server with 2x NVIDIA RTX 6000 Ada Generation (48GB) GPU, 500GB of RAM, and a CPU AMD EPYC 7453 28-Core Processor. Both servers used Ubuntu 22.04.4 LTS as a Linux distro.

We used the original split for each one of the datasets. For each experiment, which is directly related to the number of sampled nodes, we performed 4 runs and presented the mean.
It's worth noting that, to make the comparisons fairer, especially with the SLIC-based GNN (\cite{dwivedi2023benchmarking}), we trained our models under a fixed computational budget of less than $100k$ parameters. 

The model used is a 1-layer GNN with SAGEConv \cite{hamilton2017inductive} for the generalization gap analysis presented in Figures \ref{fig:gen_gap_mnist}-\ref{fig:gen_gap_fer2013} and \ref{fig:gen_gap_celeba_sb}-\ref{fig:gen_gap_pathmnist}, and the results showed in Table \ref{tab:accuracy_all}, \ref{tab:accuracy_add_exp} and \ref{tab:accuracy_vae_vs_pca}. We used PyTorch \cite{paszke2019pytorch} and, more specifically, PyTorch Geometric (PyG) framework \cite{fey2019fast} for the models.

To obtain the best embedding representation, we use Weights \& Biases (W\&B) to fine-tune the VAE’s hyperparameters. We optimize the number of layers in the CNN encoder/decoder and the latent space dimension. The CNN has three convolutional layer blocks, each with 1–5 layers, ensuring encoder-decoder symmetry. The latent dimension is chosen from {32, 64, 128, 256, 364, 512, 1024}. A grid search sweeps through the cartesian product of these configurations. The best parameters vary by dataset: MNIST, FMNIST and PathMNIST perform best with [3, 3, 1] convolutional blocks and latent dimensions of 128, 256, and 128, respectively. CIFAR10 requires [4, 5, 2] blocks and a 1024-dimensional latent space, and FER2013 and CelebA \cite{liu2015faceattributes} need [3, 3, 3] blocks, 64 and 128 latent sizes, respectively. We set the KL divergence weight
to balance regularization and reconstruction.

Tables \ref{tab:vae_hyperparams} and \ref{tab:gnn_hyperparams} summarize the hyperparameters used in our experiments for all datasets. We used Adam \nocite{kingma17-adam} optimizer with the dataset's respective parameters.

We expect to release the code of our project in the near future.


\section{Auxiliary Results for Theorem \ref{thm:ga_bound_1}}\label{thm_pf:ga_bound_1}

\subsection{Assumptions}

First, let us state a few assumptions that will be used in the following proofs.
\begin{assumption}
The convolutional maps in $\Phi_\ccalW$ are locally Lipschitz on $\ccalM$ and have norm at most $1$.
\end{assumption}

\begin{assumption}
The nonlinear function $\rho$ and its first-order derivative $\rho'$ have Lipschitz constant $1$ and $\rho(0)=0$, i.e., the function is normalized Lipschitz-continuous.
\end{assumption}

\subsection{Lemmas}

Furthermore, we need the following lemma adapted from \cite{wang2024gen_gnn}.

\begin{lemma}\label{lemma:entropy_bound}

Let $\ccalM \subset \reals^F$ be a manifold equipped with a Laplace-Beltrami (LB) operator $\ccalL$, as defined in \eqref{eqn:lb_operator}, a self-adjoint operator, whose eigenpairs are $\{\lambda_i, \phi_i\}_{i=1}^\infty$. Moreover, let $f, g \in L^2(\ccalM)$ be manifold signals over $\ccalM$, and $\ccalP_n$ the sampling operator used to sample manifold signals. Therefore, we have that:
\begin{equation}\label{eqn:graph_manfld_signal_conv}
    |\|\ccalP_nf\| - \|f\|_{\ccalM}| = \ccalO\bigg(\sqrt[4]{\dfrac{\log(1/\delta)}{n}}\bigg).
\end{equation}

\end{lemma}

\begin{proof}
    
The inner product between these signals is defined as
\begin{equation}
    \langle f, g\rangle_\ccalM = \int_\ccalM f(x)g(x)d\mu(x), 
\end{equation}
where $d\mu(x)$ is the volume element of $\ccalM$ w.r.t. its measure $\mu$. Hence, one can define the norm of such a signal as
\begin{equation}
    \|f\|^2_\ccalM = \langle f, f\rangle_\ccalM.
\end{equation}

Given that we have $\{X_1, \dots, X_N\}$ randomly sampled points from $\ccalM$, by Theorem 19 in \cite{von2008consistency} we have that
\begin{equation}
    |\langle \ccalP_Nf, \phi_i \rangle - \langle f, \phi_i \rangle_\ccalM| = \ccalO\bigg(\sqrt{\dfrac{\log(1/\delta)}{N}}\bigg).
\end{equation}

The above implies that
\begin{equation}
    |\|\ccalP_nf\|^2 - \|f\|_{\ccalM}^2| = \ccalO\bigg(\sqrt{\dfrac{\log(1/\delta)}{n}}\bigg),
\end{equation}
which further implies that
\begin{equation}\label{eqn:graph_manfld_signal_conv_2}
    |\|\ccalP_nf\| - \|f\|_{\ccalM}| \approx \ccalO\bigg(\sqrt[4]{\dfrac{\log(1/\delta)}{n}}\bigg).
\end{equation}

\end{proof}

\begin{lemma}\label{lem:ga_bound_1}
    Suppose Assumptions \ref{ass:lipschitz_filter_main_body}--\ref{ass:lipschitz_act_fn_main_body} hold. With probability at least $1-\delta$, for any GNN $\Phi_\ccalW$ as in \textbf{Setup}, we have
    \begin{align}
        |\tilde{l}(\ccalY,\Phi_\ccalW(\ccalX,\ccalL))-\tilde{l}(M_\ccalT Y_n, M_\ccalT \Phi_\ccalW(X_n,L_n))|\nonumber\\
        = \ccalO \bigg(\frac{1}{i_c} + \sqrt[d+4]{\frac{\log{1/\delta}}{n}} \bigg)
    \end{align}
where $M_\ccalT$ is the training mask [cf. \eqref{eqn:erm_ss_learning}]. 
\end{lemma}

\begin{proof}\label{lem_pf:ga_bound_1}
    We first write the difference between the loss function of the GNN and the MNN trained on the same set of parameters for the semi-supervised setting:
    \begin{align}
        &\left|\tilde{l}(M_\ccalT Y_n, M_\ccalT \Phi_\ccalW(X_n,L_n)) - \tilde{l}(\ccalY,\Phi_\ccalW(\ccalX,\ccalL))\right| \nonumber\\
        &=\dfrac{1}{p}\left|\|M_\ccalT\Phi_\ccalW(X_n, L_n) - M_\ccalT Y_n\|_2 - \|\Phi_\ccalW(\ccalX,\ccalL) - \ccalY\|_{\ccalM}\right| \nonumber\\
        &= \dfrac{1}{p}|\|M_\ccalT\Phi_\ccalW(X_n, L_n) - M_\ccalT Y_n \nonumber \\
        & \qquad + (M_\ccalT\ccalP_N\Phi_\ccalW(\ccalX,\ccalL) - M_\ccalT\ccalP_N\Phi_\ccalW(\ccalX,\ccalL))\|_2 \nonumber \\
        & \qquad - \|\Phi_\ccalW(\ccalX,\ccalL) - \ccalY\|_{\ccalM}| \nonumber\\
        &\leq \dfrac{1}{p}|\|M_\ccalT\Phi_\ccalW(X_n, L_n) - M_\ccalT\ccalP_N\Phi_\ccalW(\ccalX,\ccalL)\|_2  \label{eq:ga_bound_1_1}\\
        & \qquad + \|M_\ccalT\ccalP_N\Phi_\ccalW(\ccalX,\ccalL) - M_\ccalT\ccalP_N\ccalY\|_2 \nonumber \\
        & \qquad - \|\Phi_\ccalW(\ccalX,\ccalL) - \ccalY\|_{\ccalM}|. \nonumber
    \end{align}
    In (\ref{eq:ga_bound_1_1}) we used the fact that $Y_n = \ccalP_N\ccalY$. Now, since the training mask has unitary norm, i.e., $\|M_\ccalT\| = 1$  we have that:

    \begin{align}
        &\dfrac{1}{p}|\|M_\ccalT\Phi_\ccalW(X_n, L_n) - M_\ccalT\ccalP_N\Phi_\ccalW(\ccalX,\ccalL)\|_2  \nonumber\\
        & \qquad + \|M_\ccalT\ccalP_N\Phi_\ccalW(\ccalX,\ccalL) - M_\ccalT\ccalP_N\ccalY\|_2 \nonumber \nonumber\\
        & \qquad - \|\Phi_\ccalW(\ccalX,\ccalL) - \ccalY\|_{\ccalM}| \nonumber\\
        &\leq \underbrace{\dfrac{1}{p}|\|M_\ccalT\Phi_\ccalW(X_n, L_n) - M_\ccalT\ccalP_N\Phi_\ccalW(\ccalX,\ccalL)\|_2|}_{\circled{1}} \nonumber \\
        &\qquad+ \dfrac{1}{p}|\|\ccalP_N\Phi_\ccalW(\ccalX,\ccalL) - \ccalP_N\ccalY\|_2 - \|\Phi_\ccalW(\ccalX,\ccalL) - \ccalY\|_{\ccalM}|.\label{eq:ga_bound_1_2}
    \end{align}
    By lemma \ref{lemma:entropy_bound}, the second term on (\ref{eq:ga_bound_1_2}) has order $\ccalO((\log(1/\delta)/N)^{1/4}$. Therefore, our proof boils down to finding an upper bound to the term $\circled{1}$ above:
    \begin{align}
        &\dfrac{1}{p}|\|M_\ccalT\Phi_\ccalW(X_n, L_n) - M_\ccalT\ccalP_N\Phi_\ccalW(\ccalX,\ccalL)\|_2| \nonumber \\
        &= \dfrac{1}{p}\left[\sum_{i\in \ccalT} (\Phi_\ccalW(X_n, L_n))_i - \Phi_\ccalW(\ccalX,\ccalL)(x_i))^2\right]^{1/2} \nonumber\\
        &\leq \dfrac{1}{p}\sum_{i\in \ccalT} |(\Phi_\ccalW(X_n, L_n))_i - \Phi_\ccalW(\ccalX,\ccalL)(x_i)| \label{eq:ga_bound_1_3}\\
        &\leq \dfrac{1}{p}\cdot p \cdot|\rho((h(L_n)X_n)_i) - \rho(h(\ccalL)\ccalX(x_i))|, \label{eq:ga_bound_1_4}
    \end{align}
    where in the first inequality (\ref{eq:ga_bound_1_3}) we used the fact that, for $v \in \reals^F$, $\sum_i |v_i| \geq (\sum_i {v_i}^2)^{1/2}$, whilst in the second (\ref{eq:ga_bound_1_4}), we take the largest absolute difference between the GNN and MNN. Finally, given that the nonlinear functions $\rho$ are normalized Lipschitz continuous, we have the following bound:
    \begin{align}       
        & \dfrac{1}{p}\cdot p \cdot|\rho((h(L_n)X_n)_i) - \rho(h(\ccalL)\ccalX(x_i))| \nonumber\\
        &\leq |(h(L_n)X_n)_i - h(\ccalL)\ccalX(x_i)| \nonumber \\
        &= |(h(L_n)\ccalP_n\ccalX)_i - (\ccalP_nh(\ccalL)\ccalX)_i| \nonumber\\
        &= |[h(L_n)\ccalP_n\ccalX - \ccalP_nh(\ccalL)\ccalX]_i| \nonumber\\
        &\leq \|h(L_n)\ccalP_n\ccalX - \ccalP_nh(\ccalL)\ccalX\|_2 \nonumber\\        
        &\leq
        (C_1 + C_2)\left(\dfrac{\log(\tfrac{C_1}{\delta})}{p}\right)^{\tfrac{1}{d+4}} + 
        C_3\sqrt{\dfrac{\log(\tfrac{1}{\delta})}{p}} + \dfrac{C_4}{i_c}, \label{eq:ga_bound_1_5}
    \end{align}
    $C_1 = C_{\ccalM, 1}\tfrac{\pi^2}{6}\|\ccalX\|_\ccalM, C_2 = C_{\ccalM, 2}\tfrac{\pi^2}{6}, C_3 = \tfrac{\pi^2}{6}, C_4 = \|\ccalX\|_{\ccalM}$, where $C_{\ccalM, 1}$ and $C_{\ccalM, 2}$ are constants that depend on the dimension $d$ and the volume of the manifold.

    The last step in (\ref{eq:ga_bound_1_5}) is an adaption of the argument used in \cite{wang2024gen_gnn_asil} to prove the bound for the difference between the graph and manifold filters (Equation (51), \cite{wang2024gen_gnn_asil}).
\end{proof}


\section{Auxiliary Results for Theorem \ref{thm:MNN_learning}} \label{app:pfs_juan}

\subsection{Assumptions}

We start by stating necessary assumptions.



\begin{assumption} \label{ass:filters}
The convolutional maps in $\Phi_\ccalW$ are locally Lipschitz on $\ccalM$ and have norm at most $1$.
\end{assumption}

\begin{assumption} \label{ass:nls}
The nonlinear function $\rho$ and its first-order derivative $\rho'$ have Lipschitz constant $1$. Also, $\rho(0)=0$.
\end{assumption}

\begin{assumption} \label{ass:bl}
The convolutions in all layers of $\Phi_\ccalW$ are low-pass filters with bandwidth $c$. I.e., if $\ccalY$ is the output of a convolution, $\langle \ccalY, \phi_i\rangle = 0$ for $\lambda_i > c$, and $i_c = \argmin_i (\lambda_i - c)\boldsymbol{1}(\lambda_i  \geq c)$.
\end{assumption}

\begin{assumption} \label{ass:smpl_op_unit_norm}
    The sampling operator $\ccalP_n$ has unitary norm.
\end{assumption}

\begin{assumption} \label{ass:loss}
Let $\tbl \in \reals^{n}$ be such that $[\tbl]_{i} = n^{-1}\tilde{l}([Y]_i,[Y']_i)$ where $\tilde{l}$ is a standard loss function with Lipschitz constant $1$. The semi-supervised loss function $l$ is defined as $l(Y,Y')= n |\ccalT|^{-1} (M_\ccalT\tbl)^T \boldsymbol{1}$ where $M_\ccalT \in \{0,1\}^{|\ccalT|\times n}$ is the training mask. Since $\sigma_{\mbox{\tiny max}}(M_\ccalT) = 1$, $l$ has Lipschitz constant $n/|\ccalT|$, which is equal to $\nu^{-1}$ when $|\ccalT|=\nu n$. 
\end{assumption}

\begin{assumption} \label{ass:mnn-lipschitz-params}
	The MNN $\Phi_\ccalW(\ccalX,\ccalL)$ is $\alpha$-Lipschitz, and its gradient $ \nabla_\ccalW\Phi_\ccalW(\ccalX,\ccalL)$ is $\beta$-Lipschitz, with respect to the parameters $\ccalW$.
\end{assumption}

\subsection{Lemmas}

The following are adapted lemmas from \cite{cervino2021increase} used to prove Theorem~\ref{thm:MNN_learning}.

\begin{lemma}\label{lemma:norm_bound_dif_grad_MNN}
	Let $\Phi_\ccalW(\ccalX,\ccalL)$ be an MNN with $F_\ell=F$ for $1\leq \ell \leq \mathscr{L}-1$ and $F_\mathscr{L} = 1$. Let $\Phi(X_n,L_n)$ be a GNN with same weights $\ccalW$ on a geometric graph $G_n$ sampled uniformly from $\ccalM$ as in \eqref{eqn:geom_graph}. Under Assumptions \ref{ass:filters}-\ref{ass:smpl_op_unit_norm}, with probability $1-\delta$ it holds that
	\begin{align}
	&\|\ccalP_n\nabla_{\ccalW}\Phi_\ccalW(\ccalX,\ccalL)-\nabla_{\ccalW}\Phi_\ccalW(X_n,L_n)\| \nonumber\\
	&\leq  2 \sqrt{(\mathscr{L}-1)KF^2 + KF} \mathscr{L}^3F^{3\mathscr{L}-3} \bigg( C_1' \varepsilon + C_2'\sqrt{\frac{\log{1/\delta}}{n}} \bigg) \nonumber \\
    &\leq 2\sqrt{2(\mathscr{L}-1)K} \mathscr{L}^3F^{3\mathscr{L}-2} \bigg( C_1' \varepsilon + C_2'\sqrt{\frac{\log{1/\delta}}{n}} \bigg)\text{.}
	\end{align}
\end{lemma}
\begin{proof}
	We will first show that the gradient with respect to any arbitrary element $[W_{\ell k}]_{fg} \in \reals$ of $\ccalW$ can be uniformly bounded. Note that the maximum is attained if $\ell=\ell^\dagger=1$. Without loss of generality, assuming $\ell^\dagger>\ell-1$ and $\omega = [W_{\ell^\dagger k}]_{fg} \in \reals$, we can begin by using the output of the MNN to write
	\begin{align} 	\label{eq:norm_bound_dif__grad_WNN_first_equality}
	\|\ccalP_n\nabla_{\omega}&\Phi(\ccalX,\ccalL)-\nabla_{\omega}\Phi(X_n, L_n)\| \nonumber\\
    &\leq \|\nabla_{\omega}\ccalP_n \Phi(\ccalX,\ccalL)- \nabla_{\omega}\Phi(X_n,L_n)\| \nonumber \\
    &=\|\nabla_{\omega}[\ccalP_n\ccalX_{\mathscr{L}}]_f-\nabla_{\omega}[X_{n\mathscr{L}}]_f\| \nonumber\\
	&=\bigg\|\nabla_\omega \rho\bigg(\sum_{g=1}^{F_{\mathscr{L}-1}} \sum_{k=0}^{K-1} \ccalP_n e^{-k \ccalL} [\ccalX_{\mathscr{L}-1}]_g [W_{\mathscr{L} k}]_{fg} \bigg) \nonumber\\
    &\hspace{1.5em}-\nabla_\omega \rho \bigg(\sum_{g=1}^{F_{\mathscr{L}-1}} \sum_{k=0}^{K-1} L_n^{k} [X_{n \mathscr{L}-1}]_g[W_{\mathscr{L} k}]_{fg} \bigg)\bigg\|
	\end{align}
    where we have dropped the subscript $\ccalW$ from $\Phi$ for simplicity.
	
	Applying the chain rule and using the triangle inequality, we get
	\begin{align}
	&\|\nabla_{\omega}[\ccalP_n\ccalX_{\mathscr{L}}]_f-\nabla_{\omega}[X_{n\mathscr{L}}]_f\| \nonumber\\
    &\leq \bigg\|\bigg(\nabla \rho \bigg(\sum_{g=1}^{F_{\mathscr{L}-1}} \sum_{k=0}^{K-1} \ccalP_n e^{-k \ccalL} [\ccalX_{\mathscr{L}-1}]_g[W_{\mathscr{L} k}]_{fg}\bigg) \nonumber\\
    &\quad\qquad-\nabla \rho \bigg(\sum_{g=1}^{F_{\mathscr{L}-1}} \sum_{k=0}^{K-1}  L_n^k [X_{n \mathscr{L}-1}]_g[W_{\mathscr{L} k}]_{fg}\bigg)\bigg)  \nonumber\\
    &\qquad \times \ccalP_n\nabla_{\omega}\bigg(\sum_{g=1}^{F_{\mathscr{L}-1}} \sum_{k=0}^{K-1}  e^{-k \ccalL} [\ccalX_{\mathscr{L}-1}]_g[W_{\mathscr{L} k}]_{fg}\bigg)\bigg\| \nonumber\\	
    &\quad+\bigg\|\nabla \rho \bigg(\sum_{g=1}^{F_{\mathscr{L}-1}} \sum_{k=0}^{K-1} L_n^k [X_{n \mathscr{L}-1}]_g[W_{\mathscr{L} k}]_{fg}\bigg)  \nonumber\\
    &\qquad \times \bigg({\nabla_\omega}  \sum_{g=1}^{F_{\mathscr{L}-1}} \sum_{k=0}^{K-1} \ccalP_n e^{-k\ccalL} [\ccalX_{\mathscr{L}-1}]_g[W_{\mathscr{L} k}]_{fg} \nonumber\\
    &\qquad\quad-{\nabla_\omega} \sum_{g=1}^{F_{\mathscr{L}-1}} \sum_{k=0}^{K-1}  L_n^k [X_{n \mathscr{L}-1}]_g[W_{\mathscr{L} k}]_{fg}\bigg)\bigg\|.
	\end{align}
	
	Next, we use Cauchy-Schwarz inequality, Assumptions \ref{ass:nls} and \ref{ass:smpl_op_unit_norm}, and \cite{cervino2021increase}[Lemma 2, adapted to MNNs] to bound the terms corresponding to the gradient of the nonlinearity $\rho$ and the norm of the MNN respectively. Explicitly,
	\begin{align} \label{eqn:prop_norm_bound_dif__grad_WNN_divided_eqns} 
	&\|\nabla_{\omega}[\ccalP_n \ccalX_{\mathscr{L}}]_f-\nabla_{\omega} [X_{n\mathscr{L}}]_f\| \nonumber\\
    &\leq \bigg\|\sum_{g=1}^{F_{\mathscr{L}-1}} \sum_{k=0}^{K-1} \ccalP_n e^{-k\ccalL} [\ccalX_{\mathscr{L}-1}]_g[W_{\mathscr{L} k}]_{fg} \nonumber\\
    &\qquad-\sum_{g=1}^{F_{\mathscr{L}-1}} \sum_{k=0}^{K-1} L_n^k [X_{n \mathscr{L}-1}]_g[W_{\mathscr{L} k}]_{fg}\bigg\|F^{\mathscr{L}-1}\|\ccalX\| \nonumber\\
	&\quad+\bigg\|\sum_{g=1}^{F_{\mathscr{L}-1}} \nabla_\omega \sum_{k=0}^{K-1} \ccalP_n \bigg(e^{-k \ccalL} [\ccalX_{ \mathscr{L}-1}]_g[W_{\mathscr{L} k}]_{fg} \nonumber\\
    &\qquad- L_n^k [X_{n \mathscr{L}-1}]_g [W_{\mathscr{L} k}]_{fg})\bigg)\bigg\|
	\end{align}
    
    Applying the triangle inequality to the second term, we get
	\begin{align} \label{eqn:norm_bound_dif_grad_WNN_long_triangle_inequality} 
	&\|\nabla_{\omega}[\ccalP_n \ccalX_{\mathscr{L}}]_f - \nabla_{\omega}[X_{n\mathscr{L}}]_f\| \nonumber\\
	&\leq \bigg\|\sum_{g=1}^{F_{\mathscr{L}-1}} \sum_{k=0}^{K-1} \ccalP_n e^{-k \ccalL} [\ccalX_{\mathscr{L}-1}]_g [W_{\mathscr{L} k}]_{fg} \nonumber\\
    &\qquad-\sum_{g=1}^{F_{\mathscr{L}-1}} \sum_{k=0}^{K-1} L_n^k [X_{n\mathscr{L}-1}]_g[W_{\mathscr{L} k}]_{fg}\bigg\| F^{\mathscr{L}-1}\|\ccalX\| \nonumber\\
	&\quad+\bigg\|\sum_{g=1}^{F_{\mathscr{L}-1}} \nabla_\omega \sum_{k=0}^{K-1}  \bigg(\ccalP_n e^{-k \ccalL} [W_{\mathscr{L} k}]_{fg} \nonumber\\
    &\qquad-L_n^k \ccalP_n [W_{\mathscr{L} k}]_{fg}\bigg)[ \ccalX_{\mathscr{L}-1}]_g \bigg\| \nonumber\\
	&\quad+\sum_{g=1}^{F_{\mathscr{L}-1}}\bigg\| \nabla_\omega \sum_{k=0}^{K-1}  L_n^k  \bigg([\ccalP_n \ccalX_{ \mathscr{L}-1}]_g- [X_{n \mathscr{L}-1}]_g\bigg)[W_{\mathscr{L} k}]_{fg})\bigg\| .
	\end{align}
    
	Now note that as we consider the case in which $\ell_\dagger < \mathscr{L}-1$, using the Cauchy-Schwarz inequality we can use the same bound for the first and second terms on the right hand side of \eqref{eqn:norm_bound_dif_grad_WNN_long_triangle_inequality}. Also note that, by Assumption \ref{ass:filters}, the filters are non-expansive, which allows us to write
	\begin{align}
	&\|\nabla_{\omega}[\ccalP_n \ccalX_{\mathscr{L}}]_f-\nabla_{\omega}[X_{n\mathscr{L}}]_f\| \nonumber\\
    &\leq \bigg\|\sum_{g=1}^{F_{\mathscr{L}-1}} \sum_{k=0}^{K-1} \ccalP_n e^{-k \ccalL} [\ccalX_{\mathscr{L}-1}]_g [W_{\mathscr{L} k}]_{fg} \nonumber\\
    &\qquad-\sum_{g=1}^{F_{\mathscr{L}-1}} \sum_{k=0}^{K-1} L_n^k [X_{n\mathscr{L}-1}]_g[W_{\mathscr{L} k}]_{fg}\bigg\| F^{\mathscr{L}-1}\|\ccalX\| \nonumber\\	
    &\quad+\bigg\|\sum_{g=1}^{F_{\mathscr{L}-1}} \sum_{k=0}^{K-1}  \ccalP_n e^{-k \ccalL} [W_{\mathscr{L} k}]_{fg}	
	-L_n^k \ccalP_n [W_{\mathscr{L} k}]_{fg}\bigg\| F^{\mathscr{L}-1}\|\ccalX\| \nonumber\\
    &\quad+\sum_{g=1}^{F_{\mathscr{L}-1}}\bigg\| \nabla_{\omega}   \bigg([\ccalP_n \ccalX_{\mathscr{L}-1}]_g- [X_{n \mathscr{L}-1}]_g\bigg)\bigg\|.
	\end{align}
	The only term that remains to bound has the exact same bound derived in \eqref{eq:norm_bound_dif__grad_WNN_first_equality}, but on the previous layer $\mathscr{L}-2$. Hence, by applying the same steps $\mathscr{L}-2$ times, 
	we can obtain a bound for any element $\omega$ of tensor $\ccalH$. 
    \begin{align} \label{eqn:solved_recursion_in_lemma_2}
	&\|\nabla_{\omega}[\ccalP_n \ccalX_{\mathscr{L}}]_f-\nabla_{\omega}[X_{n\mathscr{L}}]_f\| \nonumber\\
    &\leq \mathscr{L} F^{\mathscr{L}-2} \bigg\|\sum_{g=1}^{F_{\mathscr{L}-1}} \sum_{k=0}^{K-1} \ccalP_n e^{-k \ccalL} [\ccalX_{\mathscr{L}-1}]_g [W_{\mathscr{L} k}]_{fg} \nonumber\\
    &\qquad\qquad\qquad-\sum_{g=1}^{F_{\mathscr{L}-1}} \sum_{k=0}^{K-1} L_n^k [X_{n\mathscr{L}-1}]_g[W_{\mathscr{L} k}]_{fg}\bigg\| F^{\mathscr{L}-1}\|\ccalX\| \nonumber\\	
    &\quad+\mathscr{L} F^{\mathscr{L}-2}\bigg\|\sum_{g=1}^{F_{\mathscr{L}-1}} \sum_{k=0}^{K-1}  \ccalP_n e^{-k \ccalL} [W_{\mathscr{L} k}]_{fg} \nonumber\\
    &\qquad\qquad\qquad-
    L_n^k \ccalP_n [W_{\mathscr{L} k}]_{fg}\bigg\| F^{\mathscr{L}-1}\|\ccalX\| \nonumber\\
    &\quad+\sum_{g=1}^{F_{\mathscr{L}-1}}\bigg\| \nabla_{\omega}   \bigg([\ccalP_n \ccalX_{1}]_g- [X_{n 1}]_g\bigg)\bigg\|.
	\end{align}
    
	Note that the two first terms on the right hand side can be upper bounded by Prop. \ref{prop:transf}. For the third term, the derivative of a convolutional filter at coefficient $k^{\dagger}=i$ is itself a convolutional filter with coefficients $[w_i]_{fg}$. The values of $[w_i]_{fg}$ are $1$ if $j=i$ and $0$ otherwise. Additionally, this new filter also verifies Assumption \ref{ass:filters}, as $\ccalX$ is bandlimited. Denote this filter $\Phi_{w}$. Considering that $\ell^\dagger = 1$, and using 
    \cite{wang2024gen_gnn_asil}[Prop. 2],
    \cite{von2008consistency}[Thm. 19]
    , together with the fact that $\ccalX$ is bandlimited and the triangle inequality, with probability $1-\delta$ we have
	\begin{align}\label{eqn:filter_at_1_in_lemma_2}
	&\bigg\|\Phi_w(X_{n},L_n)- \ccalP_n\Phi_w(\ccalX,\ccalL) \bigg\|	\nonumber\\
    &\leq \| \lambda_c - \lambda_{c n}\|\|\ccalX \|+\|X_{n}-\ccalP_n\ccalX \| \\
    &\leq \sqrt{F} C_{\ccalM,1} \lambda_c \varepsilon + \sqrt{F} C_3\sqrt{\frac{\log{1/\delta}}{n}}
	\end{align}
	where we have assumed each feature in $\ccalX$ has unit norm at most. Now, substituting the third term in \eqref{eqn:solved_recursion_in_lemma_2} for \eqref{eqn:filter_at_1_in_lemma_2}, and using Prop. \ref{prop:transf} for the first two terms, with probability $1-\delta$, we have
	\begin{align}\label{lemma2:last}
	&\|\nabla_{\omega}[\ccalP_n \ccalX_{\mathscr{L}}]_f-\nabla_{\omega}[X_{n\mathscr{L}}]_f\| \nonumber \\
    &\leq 2 \mathscr{L}^3F^{3\mathscr{L}-3} \bigg( C_1 \varepsilon + C_2\sqrt{\frac{\log{1/\delta}}{n}} \bigg) \nonumber\\
    &\quad+ F\sqrt{F} C_{\ccalM,1} \lambda_c \varepsilon + F\sqrt{F} C_3\sqrt{\frac{\log{1/\delta}}{n}}
    \end{align}
	To achieve the final result, note that the set $\ccalW$ has $(\mathscr{L}-1)KF^2 + KF$ elements, and each element is upper bounded by \eqref{lemma2:last}.
	\end{proof}

	\begin{lemma}\label{lemma:norm_bound_dif_loss_MNN}
	Let $\Phi_\ccalW(\ccalX,\ccalL)$ be an MNN with $F_\ell=F$ for $0 \leq \ell \leq \mathscr{L}-1$ and $F_\mathscr{L}=1$, and. 
    Let $\Phi_\ccalW(X_n,L_n)$ be a GNN with same weights $\ccalW$ on a geometric graph $G_n$ sampled uniformly from $\ccalM$ as in \eqref{eqn:geom_graph}. Under Assumptions \ref{ass:filters}--\ref{ass:loss}, with probability $1-\delta$ it holds that
	\begin{align}
	&\|\nabla_{\ccalW}l(\ccalY,\Phi_\ccalW(\ccalX,\ccalL))-\nabla_{\ccalW}l(Y_n,\Phi(X_n,L_n))\| \nonumber\\
    &\leq  2 \nu^{-1} \sqrt{(\mathscr{L}-1)KF^2 + KF} \mathscr{L}^3F^{3\mathscr{L}-3} \bigg( C_1'' \varepsilon \nonumber\\
    &\quad+ C_2''\sqrt{\frac{\log{1/\delta}}{n}} \bigg) \\
    &\leq 2 \nu^{-1}\sqrt{2(\mathscr{L}-1)K} \mathscr{L}^3F^{3\mathscr{L}-2} \bigg( C_1'' \varepsilon + C_2''\sqrt{\frac{\log{1/\delta}}{n}} \bigg)\text{.}
    \end{align}
\end{lemma}
\begin{proof}
	In order to analyze the norm of the gradient with respect to the tensor $\ccalH$, we start by taking the derivative with respect to a single element of the tensor, $\omega$, as in the proof of the previous lemma. Also as before, we drop subscript $\ccalW$ in $\Phi$. Using the chain rule to compute the gradient of the loss function $l$, we get
	\begin{align}
	&\|\nabla_{\omega}(l(\ccalP_n\ccalY,\ccalP_n\Phi(\ccalX,\ccalL))-l(Y_n,\Phi(X_n,L_n)))\| \nonumber\\
    &=\|\nabla l(\ccalP_n\ccalY,\ccalP_n\Phi(\ccalX,\ccalL))\nabla_{\omega}\ccalP_n\Phi(\ccalX,\ccalL) \nonumber\\
    &\hspace{1.5em}-\nabla l(Y_n,\Phi(X_n,L_n))\nabla_{\omega}\Phi(X_n,L_n)\|
	\end{align}
	and by the Cauchy-Schwarz and the triangle inequalities, it holds
	\begin{align}
	&\|\nabla_{\omega}(l(\ccalP_n\ccalY,\ccalP_n\Phi(\ccalX,\ccalL))-l(Y_n,\Phi(X_n,L_n)))\| \nonumber\\
	&\leq\|\nabla l(\ccalP_n\ccalY,\ccalP_n\Phi(X,\ccalL)) \nonumber\\
    &\quad-\nabla l(Y_n,\Phi(X_n, L_n))\|
    \|\nabla_{\omega}\ccalP_n\Phi(\ccalX,\ccalL)\| \nonumber\\
    &\quad+\|\nabla l(Y_n,\Phi(X_n, L_n))\| 
    \|\nabla_{\omega}\ccalP_n\Phi(\ccalX,\ccalL)-\nabla_{\omega}\Phi(X_n,L_n)\|
	\end{align}
    
	By the triangle inequality and Assumption \ref{ass:loss}, it follows
	\begin{align}
	&\|\nabla_{\omega}(l(\ccalP_n\ccalY,\ccalP_n\Phi(\ccalX,\ccalL))-\l(Y_n,\Phi(X_n,L_n)))\| \nonumber\\
    &\leq\|\nabla l(\ccalP_n \ccalY,\ccalP_n \Phi(\ccalX,\ccalL))-\nabla l(\ccalP_n \ccalY,\Phi(X_n,L_n))\|
    \nonumber\\
    &\quad\times 
    \|\nabla_{\omega}\ccalP_n \Phi(\ccalX,\ccalL)\|\|\nabla l(Y_n,\Phi(X_n,L_n)) \nonumber\\
    &\qquad\quad-\nabla\ell(\ccalP_n Y,\Phi(X_n, L_n))\| \nonumber\\
    &\quad \times \|\nabla_{\omega}\ccalP_n \Phi(\ccalX,\ccalL)\|+\|\nabla_{\omega}(\ccalP_n\Phi(\ccalX,\ccalL)-\Phi(X_n, L_n))\|  \\
	&\leq \nu^{-1}\bigg(\|Y_n-\ccalP_n\ccalY \| \nonumber\\
    &\qquad\quad+\|\Phi(X_n,L_n)-\ccalP_n\Phi(\ccalX,\ccalL) \|\bigg)
    \|\nabla_{\omega}\ccalP_n\Phi(\ccalX,\ccalL)\| \nonumber\\
    &\quad +\|\nabla_{\omega}(\ccalP_n\Phi(\ccalX,\ccalL)-\Phi(X_n,L_n))\|.
	\end{align}
	
	Next, we can use \cite{cervino2021increase}[Lemma 2, adapted to MNNs], Prop. \ref{prop:transf}, Lemma \ref{lemma:norm_bound_dif_grad_MNN}, and \cite{von2008consistency}[Thm. 19]
    to obtain 	
	\begin{align}\label{eq:last_lemma3}
	&\|\nabla_{\omega}(l(\ccalP_n\ccalY,\ccalP_n\Phi(\ccalX,\ccalL))-l(Y_n,\Phi(X_n,L_n)))\| \nonumber\\
	&\leq \nu^{-1}\bigg( C_5 \sqrt{\frac{\log{1/\delta}}{n}} \nonumber\\
    &\qquad\quad+ \mathscr{L} F^{\mathscr{L}-2} \bigg( C_1 \varepsilon + C_2\sqrt{\frac{\log{1/\delta}}{n}} \bigg) \bigg) F^{\mathscr{L}-1}\sqrt{F}  \nonumber\\
    &\quad+ 2 \nu^{-1}\mathscr{L}^3F^{3\mathscr{L}-3} \bigg( C_1' \varepsilon + C_2'\sqrt{\frac{\log{1/\delta}}{n}} \bigg) 
    \end{align}
	where we also assume $\|\ccalX\|\leq \sqrt{F}$.

    Finally, when $\tilde{l}$ is the $2$-norm we can use \cite{von2008consistency}[Thm. 19] to show: 
    \begin{align}
        &\|\nabla_{\omega}(l(\ccalY,\Phi(\ccalX,\ccalL))-l(Y_n, \Phi(X_n,L_n)))\| \nonumber\\
        &\leq \|\nabla_{\omega}(l(\ccalY,\Phi(\ccalX,\ccalL))-l(\ccalP_n\ccalY,\ccalP_n\Phi(\ccalX,\ccalL)))\| \nonumber\\
        &\quad+ \|\nabla_{\omega}(l(\ccalP_n\ccalY,\ccalP_n\Phi(\ccalX,\ccalL))-l(Y_n,\Phi(X_n,L_n)))\| \\
        &\leq \nu^{-1}\bigg( \tilde{C}_5 \sqrt{\frac{\log{1/\delta}}{n}} \nonumber\\ 
        &\qquad\quad+ \mathscr{L} F^{\mathscr{L}-2} \bigg( C_1 \varepsilon + \tilde{C}_2\sqrt{\frac{\log{1/\delta}}{n}} \bigg) \bigg) F^{\mathscr{L}-1}\sqrt{F}  \nonumber\\
        &\quad+ 2 \nu^{-1}\mathscr{L}^3F^{3\mathscr{L}-3} \bigg( C_1' \varepsilon + \tilde{C}_2'\sqrt{\frac{\log{1/\delta}}{n}} \bigg) \text{.}
    \end{align}
    
	Noting that tensor $\ccalW$ has $(\mathscr{L}-1)KF^2 + KF$ elements, and that each individual term can be bounded by \eqref{eq:last_lemma3}, we arrive at the desired result.
\end{proof}	


\begin{proposition} \label{prop:expected_grad}
	Consider the ERM problem in \eqref{eqn:erm_ss_learning} and let $\Phi_\ccalW(\ccalX,\ccalL)$ be an $\mathscr{L}$-layer MNN with $F_{\ell}=F$ for $0 \leq \ell \leq L-1$ and $F_{\mathscr{L}}=1$. 
    Let $\Phi(X_n,L_n)$ be a GNN with same weights $\ccalW$ on a geometric graph $G_n$ sampled uniformly from $\ccalM$ as in \eqref{eqn:geom_graph}. Under Assumptions \ref{ass:filters}--\ref{ass:loss}, it holds
	\begin{align} \label{eqn:expected_grad}
	&\mbE[\|\nabla_{\ccalW}l(\ccalY,\Phi_\ccalW(\ccalX,\ccalL))-\nabla_{\ccalW}l(Y_n,\Phi_\ccalW(X_n,L_n))\|] \nonumber\\ 
    &= \ccalO\bigg(\gamma\bigg(\varepsilon+\sqrt{\frac{\log{n}}{n}}\bigg)\bigg)
	\end{align}
	where $\gamma$ is a constant that depends on the number of layers $L$, features $F$, and filter taps $K$ of the GNN.
\end{proposition}
\begin{proof}
To start, consider the event $A_n$ such that
	\begin{align}
    &\|\nabla_{\ccalW}l(\ccalP_n\ccalY,\ccalP_n\Phi(\ccalX,\ccalL))-\nabla_{\ccalW}l(Y_n,\Phi(X_n,L_n))\| \nonumber\\
    &\leq 2\sqrt{2(\mathscr{L}-1)K} \mathscr{L}^3F^{3\mathscr{L}-2} \bigg( C_1'' \varepsilon + C_2''\sqrt{\frac{\log{1/\delta}}{n}} \bigg)
	\end{align}
	where we have dropped the subscript $\ccalW$ where it is clear from context. Taking the disjoint events $A_n$ and $A_n^c$, and denoting the indicator function $\bbone(\cdot)$, we split the expectation as
	\begin{align}
	&\mbE[\|\nabla_{\ccalW}l(\ccalY,\Phi(\ccalX,\ccalL))-\nabla_{\ccalW}l(Y_n,\Phi(X_n, L_n))\|] \nonumber\\
    &=\mbE[\|\nabla_{\ccalW}(l( \ccalY, \Phi(\ccalX,\ccalL))-\l(Y_n,\Phi(X_n,L_n)))\|\bbone(A_n)] \nonumber\\
	&\quad+\mbE[\|\nabla_{\ccalW}l( \ccalY, \Phi(\ccalX,\ccalL)) -\nabla_{\ccalW}l(Y_n,\Phi(X_n,L_n))\|\bbone(A_n^c)]\label{eqn:Theo1_expectation_split}
	\end{align}
    We can then bound the term corresponding to $A_n^c$ using the chain rule, the Cauchy-Schwarz inequality, Assumption \ref{ass:loss}, and \cite{cervino2021increase}[Lemma 2, adapted to MNNs] as follows 
	\begin{align}
	&\|\nabla_{\ccalW}l( \ccalY, \Phi(\ccalX,\ccalL))-\nabla_{\ccalW}l(Y_n,\Phi(X_n,L_n))\|\nonumber\\
	&\leq \|\nabla_{\ccalW}l( \ccalY, \Phi(\ccalX\ccalL))\|+\|\nabla_{\ccalW}l(Y_n,,L_n))\| \\
	&\leq \|\nabla l( \ccalY, \Phi(\ccalX\ccalL))\| \|\nabla_{\ccalW} \Phi(\ccalX,\ccalL)\| \nonumber\\
    &\hspace{1em}+\|\nabla l(Y_n,\Phi(X_n, L_n))\|\|\nabla_{\ccalW}\Phi(X_n, L_n)\|\\
	&\leq \|\nabla_{\ccalW} \Phi(\ccalX,\ccalL)\|+\|\nabla_{\ccalW}\Phi(X_n,L_n)\|\\
	&\leq 2 F^{\mathscr{L}} \sqrt{(\mathscr{L}-1)KF+K}. \label{eqn:Theo1_bound_A_c}
	\end{align}
	Going back to \eqref{eqn:Theo1_expectation_split}, we can substitute the bound obtained in \eqref{eqn:Theo1_bound_A_c}, take $P(A_n)=1-\delta$, and use Lemma \ref{lemma:norm_bound_dif_loss_MNN} to get
	\begin{align}
	&\mbE[\|\nabla_{\ccalW}l( \ccalY, \Phi(\ccalX,\ccalL))-\nabla_{\ccalW}l(Y_n,\Phi(X_n, L_n))\|] \nonumber\\
    &\leq \delta 2 F^{\mathscr{L}} \sqrt{(\mathscr{L}-1)KF+K} \nonumber\\
    &\quad+(1-\delta) 2\nu^{-1} \sqrt{2(\mathscr{L}-1)K} \mathscr{L}^3F^{3\mathscr{L}-2} \bigg( C_1'' \varepsilon + C_2''\sqrt{\frac{\log{1/\delta}}{n}} \bigg) \text{.}
	\end{align}
	Setting $\delta=\frac{1}{\sqrt{n}}$ completes the proof. 
\end{proof}


\begin{lemma}\label{lemma:martingale}
		Consider the ERM problem in \eqref{eqn:erm_ss_learning} and let $\Phi_\ccalW(\ccalX,\ccalL)$ be an $\mathscr{L}$-layer MNN with $F_{\ell}=F$ for $0 \leq \ell \leq L-1$ and $F_{\mathscr{L}}=1$. 
        Fix $\epsilon>0$ and step size $\eta<{\theta^{-1}}$, with {$\theta=\alpha+\beta F\sqrt{2K(\mathscr{L}-1)} $}. Let $\Phi(X_n,L_n)$ be a GNN with same weights $\ccalW$ on a geometric graph $G_n$ sampled uniformly from $\ccalM$ as in \eqref{eqn:geom_graph}. Consider the iterates generated by \red{\eqref{eqn:GNN_Learning_Step}}. Under Assumptions \ref{ass:filters}--\ref{ass:mnn-lipschitz-params}, if at step $k$ of epoch $e$ the number of nodes $n(e)$ verifies
		\begin{align}
	      &\mbE[\|\nabla_{\ccalW}l( \ccalY, \Phi_{\ccalW_k}(\ccalX,\ccalL))-\nabla_{\ccalW}l(Y_{n(e)},\Phi_{\ccalW_k}(X_{n(e)},L_{n(e)}))\|] \nonumber\\
          &\leq \|\nabla_{\ccalW} l( \ccalY, \Phi_{\ccalW_k}(\ccalX,\ccalL))\|
		\end{align}
        then the iterate generated by graph learning step \red{\eqref{eqn:GNN_Learning_Step}} satisfies
        \begin{align}
        \mbE[l( \ccalY, \Phi_{\ccalW_{k+1}}(\ccalX,\ccalL))] \leq l( \ccalY,\Phi_{\ccalW_k}(\ccalX,\ccalL)).
        \end{align}
    \end{lemma}
\begin{proof}
	To start, we do as in \cite{bertsekas2000gradient}, i.e., we define a continuous function $g(\epsilon)$ that at $\epsilon=1$ takes the value of the loss function on $\ccalM$ at iteration $k+1$, and at $\epsilon=0$, the value at iteration $k$. Explicitly,
	\begin{align}
	g(\epsilon)=l( \ccalY, \Phi_{\ccalW_{k}-\epsilon\eta_k \nabla_{\ccalW} l(Y_n,\Phi_{\ccalW_k}(X_n,L_n))}(\ccalX, \ccalL)). 
	\end{align}
    
	Function $g(\epsilon)$ is evaluated on the manifold data $\ccalY,\ccalX,\ccalL$, but the steps are controlled by the graph data $Y_n,X_n,L_n$. Applying the chain rule, the derivative of $g(\epsilon)$ with respect to $\epsilon$ can be written as
	\begin{align}
	&\frac{\partial g(\epsilon)}{\partial \epsilon}
    =-\eta_k \nabla_{\ccalW} l(Y_n,\Phi_{\ccalW_k}(X_n,L_n)) \nonumber\\
    &\hspace{4.5em} \times \nabla_{\ccalW}l( \ccalY, \Phi_{\ccalW_{k}-\epsilon\eta_k \nabla_{\ccalW} l(Y_n,\Phi_{\ccalW_k}(X_n,L_n))}(\ccalX,\ccalL)).
	\end{align}
    
    Between iterations $k+1$ and $k$, the difference in the loss function $l$ can be written as the difference between $g(1)$ and $g(0)$,
	\begin{align}
	g(1)-g(0)=l( \ccalY,\Phi_{\ccalW_{k+1}}(\ccalX,\ccalL))-l( \ccalY, \Phi_{\ccalW_k}(\ccalX,\ccalL)).
	\end{align}
    
	Integrating the derivative of $g(\epsilon)$ in $[0,1]$, we get
	\begin{align}
	&l( \ccalY, \Phi_{\ccalW_{k+1}}(\ccalX,\ccalL))-l( \ccalY, \Phi_{\ccalW_k}(\ccalX,\ccalL)) 
    = g(1)-g(0) \nonumber\\
    &=\int_0^1 \frac{\partial g(\epsilon)}{\partial \epsilon} d\epsilon \nonumber\\
	&=-\int_0^1\eta_k \nabla_{\ccalW} l(Y_n,\Phi_{\ccalW_k}(X_n,L_n)) \nonumber\\
    &\qquad \times \nabla_{\ccalW}l( \ccalY, \Phi_{\ccalW_k-\epsilon\eta_k \nabla_{\ccalW} l(Y_n,\Phi_{\ccalW_k}(X_n,L_n))}(\ccalX,\ccalL))d\epsilon.
	\end{align}
    
	Now note that the last term of the integral does not depend on $\epsilon$. Thus, we can proceed by adding and subtracting $\nabla_{\ccalH}l(Y,\Phi(\ccalH_{k},\ccalL,X))$ inside the integral to get
	\begin{align}
	&l( \ccalY, \Phi_{\ccalW_{k+1}}(\ccalX,\ccalL))-l(Y,\Phi(X;\ccalH_{k},\ccalL)) \nonumber\\
    &=-\eta_k \nabla_{\ccalW} l(Y_n,\Phi_{\ccalW_{k}}(X_n,L_n)) \nonumber\\
    &\quad\times \int_0^1 \nabla l( \ccalY, \Phi_{\ccalW_{k}-\epsilon\eta_k \nabla l(Y_n,\Phi_{\ccalW_k}(X_n,L_n))}(\ccalX,\ccalL)) \nonumber\\
    &\qquad+\nabla_{\ccalW}l( \ccalY, \Phi_{\ccalW_k}(\ccalX,\ccalL))-\nabla_{\ccalW}l( \ccalY, \Phi_{\ccalW_k}(\ccalX,\ccalL))d\epsilon \nonumber\\
	&=-\eta_k \nabla_{\ccalW} l(Y_n,\Phi_{\ccalW_k}(X_n,L_n))\nabla_{\ccalW}l( \ccalY, \Phi_{\ccalW_k}(\ccalX,\ccalL)) \nonumber\\
    &\quad-\eta_k \nabla_{\ccalW} l(Y_n,\Phi_{\ccalW_k}(X_n,L_n)) \nonumber\\
    &\qquad\times\int_0^1 \nabla_{\ccalW}l( \ccalY, \Phi_{\ccalW_{k}-\epsilon\eta_k \nabla l(Y_n,\Phi_{\ccalW_k}(X_n,L_n))}(\ccalX,\ccalL)) \nonumber\\
    &\qquad\qquad-\nabla l( \ccalY, \Phi_{\ccalW_k}(\ccalX,\ccalL))d\epsilon. \label{eqn:Lemma1_before_bouding_integral}
	\end{align}
	
	Next, we can apply the Cauchy-Schwarz inequality to the last term of \eqref{eqn:Lemma1_before_bouding_integral} and take the norm of the integral (which is smaller than the integral of the norm), to obtain
	\begin{align}
	&l( \ccalY, \Phi_{\ccalW_{k+1}}(\ccalX,\ccalL))-l( \ccalY, \Phi_{\ccalW_k}(\ccalX,\ccalL)) \nonumber\\
    &\leq-\eta_k \nabla_{\ccalW} l(Y_n,\Phi_{\ccalW_k}(X_n,L_n))\nabla_{\ccalW}l( \ccalY, \Phi_{\ccalW_k}(\ccalX,\ccalL)) \nonumber\\
	&\quad+\eta_k\|\nabla_{\ccalW} l(Y_n,\Phi_{\ccalW_k}(X_n,L_n))\| \nonumber\\
    &\qquad\times\int_0^1\|\nabla_{\ccalW}l( \ccalY, \Phi_{\ccalW_k}(\ccalX,\ccalL)) \nonumber\\
    &\qquad\qquad-\nabla l(\ccalY, \Phi_{\ccalW_{k}-\epsilon\eta_k \nabla l(Y_n,\Phi_{\ccalW_k}(X_n,L_n))}(\ccalX,\ccalL))\|d\epsilon. \end{align}
	
	By \cite{cervino2021increase}[Lemma 6, adapted to MNNs], we use $\theta$ to write
	\begin{align}
	&l(\ccalY, \Phi_{\ccalW_{k+1}}(\ccalX,\ccalL))-l( \ccalY, \Phi_{\ccalW_k}(\ccalX,\ccalL)) \nonumber\\
    &\leq-\eta_k \nabla_{\ccalW} l(Y_n,\Phi_{\ccalW_k}(X_n,Ln))\nabla_{\ccalW}l( \ccalY, \Phi_{\ccalW_k}(\ccalX,\ccalL)) \nonumber\\
    &\quad+\theta\eta_k \|\nabla_{\ccalW} l(Y_n,\Phi_{\ccalW_k}(X_n,L_n))\| \nonumber\\
    &\qquad \times \int_0^1 \bigg\| \eta_k \nabla_{\ccalW} l(Y_n,\Phi_{\ccalW_n}(X_n,L_n))\bigg\|\epsilon d\epsilon\\
	&\leq-\eta_k \nabla_{\ccalW} l(Y_n,\Phi_{\ccalW_k}(X_n,L_n))\nabla_{\ccalW}l( \ccalY, \Phi_{\ccalW_k}(\ccalX,\ccalL)) \nonumber\\ &\quad+\frac{\eta_k^2 \theta}{2} \|\nabla_{\ccalW} l(Y_n,\Phi_{\ccalW_k}(X_n,L_n))\|^2. 
	\end{align}
	
	Factoring out $\eta_k$, we get
	\begin{align}
	&l(\ccalY, \Phi_{\ccalW_{k+1}}(\ccalX,\ccalL))-l( \ccalY, \Phi_{\ccalW_k}(\ccalX,\ccalL)) \nonumber\\
    &\leq -\frac{\eta_k}{2}  \bigg(-\|\nabla_{\ccalW} l(Y_n,\Phi_{\ccalW_k}(X_n,L_n))\|^2 \nonumber\\ 
    &\quad+2\nabla_{\ccalW}l(Y_n,\Phi_{\ccalW_k}(X_n,L_n))^T\nabla_{\ccalW}l( \ccalY, \Phi_{\ccalW_k}(\ccalX,\ccalL))\bigg) \nonumber\\
	&\quad+\frac{\eta_k^2 \theta-\eta_k}{2} \|\nabla_{\ccalW} l(Y_n,\Phi_{\ccalW_k}(X_n,L_n))\|^2.
	\end{align}
	
	Given that the norm is induced by the vector inner product in Euclidean space, for any two vectors $A,B$, $||A-B||^2-||B||^2=||A||^2-2\langle A,B\rangle$. Hence,
	\begin{align}
	&l(\ccalY, \Phi_{\ccalW_{k+1}}(\ccalX,\ccalL))-l( \ccalY, \Phi_{\ccalW_k}(\ccalX,\ccalL)) \nonumber\\
	&\leq -\frac{\eta_k}{2} \big(\|\nabla_{\ccalW} l( \ccalY, \Phi_{\ccalW_k}(\ccalX,\ccalL))\|^2 \nonumber\\
    &\quad -||\nabla_{\ccalW}l(Y_n,\Phi_{\ccalW_k}(X_n,L_n))-\nabla_{\ccalW}l( \ccalY, \Phi_{\ccalW_k}(\ccalX,\ccalL))||^2\big) \nonumber\\
	&\quad+\frac{\eta_k^2 \theta-\eta_k}{2} \|\nabla_{\ccalW} l(Y_n,\Phi_{\ccalW_k}(X_n,L_n))\|^2.
	\end{align}
	Considering the first term on the right-hand side, we know that the norm of the expected difference between the gradients is bounded by 
    Prop. \ref{prop:expected_grad}.
    Given that norms are positive, the inequality still holds when the elements are squared (if $a>b,a\in \reals_+,b\in\reals_+$, then $a^2>b^2$). Considering the second term on the right hand side, we impose the condition that $\eta_k<\frac{1}{\theta}$, which makes this term negative. Taking the expected value over all the nodes completes the proof.
\end{proof}